\documentclass{article}

\usepackage{arxiv}

\usepackage[utf8]{inputenc}
\usepackage[T1]{fontenc}
\usepackage{hyperref}
\usepackage{url}
\usepackage{booktabs}
\usepackage{amsfonts}
\usepackage{amsmath}
\usepackage{amsthm}
\usepackage{algorithm}
\usepackage{algpseudocode}
\usepackage[numbers]{natbib}
\usepackage{doi}
\usepackage{graphicx}
\usepackage{comment}
\usepackage{multirow}
\usepackage{amssymb}

\theoremstyle{definition}
\newtheorem{definition}{Definition}
\theoremstyle{remark}
\newtheorem{remark}{Remark}

\theoremstyle{plain}
\newtheorem{theorem}{Theorem}
\newtheorem{lemma}{Lemma}
\newtheorem{corollary}{Corollary}
\newtheorem{assumption}{Assumption}
\newtheorem{proposition}{Proposition}

\newcommand{\N}{\mathbb{N}}

\newcommand{\Prob}{\mathbb{P}}
\newcommand{\E}{\mathbb{E}}

\title{On the Probability of First Success in Differential Evolution: Hazard Identities and Tail Bounds}

\renewcommand{\shorttitle}{Hazard Bounds in Differential Evolution}

\date{June, 2025}

\author{%
  \href{https://orcid.org/0009-0001-7212-9167}{\includegraphics[scale=0.06]{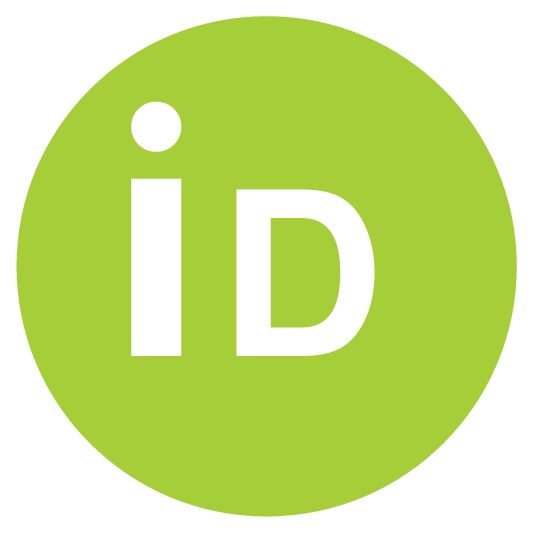}\hspace{1mm}Dimitar Nedanovski}\\
  Faculty of Mathematics and Informatics\\
  Sofia University St. Kliment Ohridski\\
  James Baucher Blvd., 1164 Sofia, Bulgaria\\
  \texttt{dnedanovski@gmail.com}
  \And
  \href{https://orcid.org/0000-0002-1444-8522}{\includegraphics[scale=0.06]{orcid.eps}\hspace{1mm}Svetoslav Nenov}\\
  Department of Mathematics\\
  University of Chemical Technology and Metallurgy\\
  Sofia, Bulgaria\\
  \texttt{nenov@uctm.edu}
  \And
  Dimitar Pilev\\
  Department of Informatics\\
  University of Chemical Technology and Metallurgy\\
  Sofia, Bulgaria\\
  \texttt{pilev@uctm.edu}
}

\hypersetup{
  pdftitle={On the Probability of First Success in Differential Evolution: Hazard Identities and Tail Bounds},
  pdfsubject={cs.NE, cs.AI},
  pdfauthor={Dimitar Nedanovski, Svetoslav Nenov, Dimitar Pilev},
  pdfkeywords={Differential Evolution, First-Hitting Time, Conditional Hazard, Success Probability, Tail Bounds, Plateau Functions, Stochastic Optimization},
}

\begin{document}
\maketitle

\begin{abstract}
We study first-hitting times in Differential Evolution (DE) through a conditional hazard framework.
Instead of analyzing convergence via Markov-chain transition kernels or drift arguments, we express
the survival probability of a measurable target set $A$ as a product of conditional first-hit
probabilities (hazards) $p_t=\Prob(E_t\mid\mathcal F_{t-1})$.
This yields distribution-free identities for survival and explicit tail bounds whenever
deterministic lower bounds on the hazard hold on the survival event.

For the L-SHADE algorithm with current-to-$p$best/1 mutation, we construct a checkable
algorithmic witness event $\mathcal L_t$ under which the conditional hazard admits an explicit
lower bound depending only on sampling rules, population size, and crossover statistics.
This separates theoretical constants from empirical event frequencies and explains why
worst-case constant-hazard bounds are typically conservative.

We complement the theory with a Kaplan--Meier survival analysis on the CEC2017 benchmark suite.
Across functions and budgets, we identify three distinct empirical regimes:
(i) strongly clustered success, where hitting times concentrate in short bursts;
(ii) approximately geometric tails, where a constant-hazard model is accurate; and
(iii) intractable cases with no observed hits within the evaluation horizon.
The results show that while constant-hazard bounds provide valid tail envelopes, the practical
behavior of L-SHADE is governed by burst-like transitions rather than homogeneous per-generation
success probabilities.
\end{abstract}



\section{Introduction}

Differential Evolution (DE) is a widely used population-based algorithm for continuous optimization. Adaptive variants such as L-SHADE have demonstrated strong empirical performance across diverse problem classes, yet theoretical understanding of their convergence behavior remains limited. A key difficulty is that these algorithms are inherently non-stationary: the population size, parameter distributions, and archive contents all evolve dynamically during optimization, making standard analytical tools difficult to apply.

In this paper we study L-SHADE in terms of {\it first-hitting times}.  Let
$$
A_\varepsilon = \{x\in [l,u]^d: f(x)\le f^\star + \varepsilon\},\qquad l<u,\ \varepsilon>0
$$
denote $\varepsilon$-sublevel set of the global minimum $f^\star$.

Let $\tau_A$ denote the first generation at which the population contains a point in $A$.
When runs are terminated by a budget of $B$ function evaluations (NFEVs), we also consider the evaluation-indexed
hitting time $\tau_A^{\mathrm{eval}}$ and the induced maximal generation $n(B)$ determined by the (possibly
time-varying) population size.

Our starting point is an exact identity for the survival probability
$\mathbb{P}(\tau_A>n)$ in terms of \emph{conditional first-hit hazards}.
Let $E_t$ denote the event that the first hit of $A$ occurs at generation $t$, and let
$\mathcal{F}_{t-1}$ be the post-generation filtration encoding the full algorithmic history up to time $t-1$.
Define the conditional first-hit probability (hazard)
$$
p_t = \mathbb{P}(E_t \mid \mathcal{F}_{t-1}), \qquad t\ge 1,
$$
and the survival event $H_{t-1}=\{\tau_A>t-1\}$.
Without any independence or stationarity assumptions, the survival function admits a product representation
in terms of the survival-only hazards $h_t=\mathbb{P}(E_t\mid H_{t-1})$.
This yields distribution-free tail bounds whenever one can lower bound $h_t$ on survival.

To turn this viewpoint into a usable bound for L-SHADE, we construct a measurable \emph{witness event} $L_t$
(depending only on the state at time $t-1$) that certifies a favorable configuration for
current-to-$p$best/1 with binomial crossover.  On $H_{t-1}\cap L_t$ we show that the conditional hazard admits a
deterministic lower bound of the form
$$
p_t \ge a_t > 0,
$$
where $a_t$ is explicit and depends only on (i) discrete sampling rules (index choices and memory slot selection),
(ii) population/archive sizes, and (iii) crossover tail probabilities.
Crucially, this separates the one-step success probability into:
\begin{enumerate}
\item an \emph{explicit theoretical factor} $a_t>0$ (typically small and conservative), and
\item an \emph{empirical frequency factor} describing how often the witness configuration occurs along the run.
\end{enumerate}
Writing
$$
\gamma_t = \mathbb{P}(L_t \mid H_{t-1}),
$$
we obtain a survival-only hazard lower bound of the form
$$
h_t \;=\; \mathbb{P}(E_t \mid H_{t-1})
\;\ge\; a_t\,\gamma_t,
$$
and consequently an explicit survival envelope
$$
\mathbb{P}(\tau_A>n)
\;\le\; \mathbb{P}(E_0^c)\prod_{t=1}^n(1-a_t\gamma_t)
\;\le\; \mathbb{P}(E_0^c)\exp\!\Bigl(-\sum_{t=1}^n a_t\gamma_t\Bigr).
$$
This ``constant $\times$ frequency'' decomposition makes precise why worst-case constant-hazard bounds can be
valid yet pessimistic: even if $a_t$ is tiny, the cumulative intensity $\sum_{t\le n} a_t\gamma_t$ can become large
when $L_t$ occurs frequently (often in short clusters).

\paragraph{A three-step scheme for understanding success within a finite budget.}
The paper is organized around the following scheme, which directly targets success within a finite evaluation budget:
\begin{enumerate}
\item \textbf{Empirical success-by-budget (good functions).}
For a fixed budget $B$ NFEVs we study the success probability
$\mathbb{P}(\tau_{A_\varepsilon}^{\mathrm{eval}}\le B)$.
Functions for which this probability is close to $1$ are ``good'' at budget $B$ in the practical sense that
a finite NFEV budget suffices with high probability.

\item \textbf{Witness frequency on survival.}
For each generation $t$ we quantify (and, in principle, estimate from repeated runs) the survival-conditional
frequency $\gamma_t=\mathbb{P}(L_t\mid H_{t-1})$, i.e., how often L-SHADE enters a configuration where the
one-step hazard admits an explicit certified lower bound.

\item \textbf{Certification via the hazard product.}
Combining $h_t\ge a_t\gamma_t$ with the exact survival product yields a certified tail bound for
$\mathbb{P}(\tau_{A_\varepsilon}>n)$ and, via $n=n(B)$, for $\mathbb{P}(\tau_{A_\varepsilon}^{\mathrm{eval}}>B)$.
Thus, for ``good'' functions, high success probability within a finite budget is explained by the build-up of
the cumulative witness intensity $\sum_{t\le n(B)} a_t\gamma_t$.
\end{enumerate}

\paragraph{Empirical survival regimes.}
To connect the theory with practice we perform a Kaplan--Meier survival analysis on the CEC2017 benchmark suite
in dimension $d=10$ under standard L-SHADE parameters and two evaluation budgets.
The survival curves reveal three distinct empirical regimes: (i) strongly clustered success, where the survival
function collapses in short bursts; (ii) approximately geometric tails, where a constant-hazard model is accurate;
and (iii) intractable cases with no observed hits within the evaluation horizon.
These results support the central message: mathematically valid constant-hazard envelopes are often conservative
because L-SHADE's practical success is governed by configuration-dependent transitions rather than homogeneous
per-generation progress.

\paragraph{Contributions.}
Our contributions are: (i) a hazard-based formulation of first-hitting times for DE yielding exact survival identities
and distribution-free tail bounds; (ii) a concrete witness event $L_t$ for L-SHADE/current-to-$p$best/1 that yields an
explicit hazard floor on survival; and (iii) an empirical survival study (Kaplan--Meier and geometric envelopes) that
characterizes clustered versus near-geometric success regimes and explains the observed conservatism of worst-case
constant-hazard bounds.

\paragraph{Supplementary Material.}
Supplementary Material is publically available at
\url{https://github.com/snenovgmailcom/lshade\_hazard\_project}
online and contains additional empirical analyses, including cluster geometry at hitting times and Kaplan--Meier validation of hazard bounds.

\section{Success Events and Probability Space}\label{sec:probspace}

The constructions in this section introduce the notations. One may find complete standard constructions of this type for example in \cite[Section~2.1.4, Section~4.2]{Durrett2019}, \cite{Shalizi2006IonescuTulcea}, \cite[Section~13.1]{Dudley1989}, \cite{Bayer2011AdvancedProb}, \cite[Section 33]{Billingsley1986}.

Let $[l,u]^d$ be the search cuboid, and let
$\left\{P^{(t)}\right\}_{t\ge 0}$ denote the population process produced by the DE variant under study,
where
$$
P^{(t)}=\left\{{\boldsymbol x}^{(t)}_1,\dots,{\boldsymbol x}^{(t)}_{\kappa_t}\right\}\subset [l,u]^{\kappa_t}
$$
and $\kappa_t$ is the (possibly time-varying) population size.

We model the algorithm on a filtered probability space
$\left(\Omega,\mathcal{F},\left\{\mathcal{F}_t\right\}_{t\ge 0},\Prob\right)$, where $\mathcal{F}_t$ represents the
information available \emph{after completing generation $t$} (including all random draws used
in that generation and all derived state variables such as the population, archive, and parameter memories).
A concrete probability product-space construction supporting full adaptivity is given in
Appendix~\ref{app:probspace}.

Let $S_t$ denote the full algorithmic state after generation $t$ (e.g.\ population, objective values,
archive, parameter memories). We take the natural (post-generation) filtration
$$
\mathcal{F}_t=\sigma\left(S_0,\dots,S_t\right),\qquad t\ge 0.
$$
Fix a measurable target set $A\subset [l,u]^d$ with positive Lebesgue measure (usually $\varepsilon$ neighborhood of the global minimum set of objective function $f$ or corresponding level definition).

Define the disjoint family of {\it first-hit events}
\begin{align*}
E_0 &= \left\{(\exists i)\,x^{(0)}_i\in A\right\}\in\mathcal{F}_0,\\
E_t &= \left\{(\forall s\le t-1)(\forall i)\,x^{(s)}_i\notin A\ \text{and}\ (\exists i)\,x^{(t)}_i\in A\right\}\in\mathcal{F}_t,\qquad t\ge 1,
\end{align*}
and the {\it first-hitting time}
$$
\tau_A = \inf\{t\ge 0:\ E_t\ \text{occurs}\}
        = \inf\{t\ge 0:\ (\exists i)\,x^{(t)}_i\in A\}.
$$
Let the survival events be
$$
H_t = \{\tau_A>t\} = \bigcap_{s=0}^t E_s^c \in \mathcal{F}_t,\qquad t\ge 0.
$$

We define the {\it conditional first-hit probabilities (hazards)} by
$$
p_0=\Prob(E_0), \qquad p_t = \Prob(E_t\mid\mathcal{F}_{t-1}),\quad t\ge 1
$$
and {\it survival-only hazard} $h_t = \Prob(E_t \mid H_{t-1})$, $t\geq 1$. Obviously $H_{t-1}\in \mathcal{F}_{t-1}$, and
$$h_t = \E[p_t\mid H_{t-1}]= \E[\Prob(E_t\mid\mathcal{F}_{t-1})\mid H_{t-1}] = \E[p_t\mid H_{t-1}].$$

Also since $E_t = H_{t-1}\cap\left\{(\exists i)\,x^{(t)}_i\in A\right\}$, then
\begin{equation}\label{eq:hazard-on-survival}
p_t = \mathbf{1}_{H_{t-1}}\,
\Prob\left((\exists i)\,x^{(t)}_i\in A \,\big|\, \mathcal{F}_{t-1}\right),
\qquad t\ge 1,
\end{equation}
so $p_t=0$ automatically off the survival event.

\section{Conditional Success}\label{sec:conditional-success}

All identities and bounds in current section hold for any arbitrary measurable target set $A$.

\begin{lemma}{(see \cite{CoxOakes1984})}\label{lem:failure-identity}
For every $n \ge 0$,
\begin{gather}
\mathbf{1}_{\{\tau_A > n\}} = \prod_{t=0}^n \mathbf{1}_{E_t^c}, \quad
\mathbb{P}(\tau_A > n) = \mathbb{E}\left[\prod_{t=0}^n \mathbf{1}_{E_t^c}\right].
\end{gather}
Moreover, for every $n\ge 1$
\begin{align}
\mathbb{P}(\tau_A > n \mid \mathcal{F}_{n-1}) =& \mathbf{1}_{H_{n-1}}(1 - p_n),\label{eq:one-step}\\
\label{eq:recursion}
\mathbb{P}(\tau_A > n) =& \mathbb{E}\left[\mathbf{1}_{H_{n-1}}(1 - p_n)\right], \quad n\geq 1.
\end{align}
\end{lemma}

\begin{lemma}(see \cite{CoxOakes1984})\label{lem:upper-product}
For every $n\ge0$,
$$
\Prob(\tau_A>n)\le \E\Big[\prod_{t=0}^n (1-p_t)\Big].
$$
\end{lemma}

The exponential envelope in Theorem~\ref{thm:hazard-bounds} is a classical
discrete-time survival bound: if conditional hazards $\gamma_t$ are given,
then $\Prob(\tau>n)\le\prod_{t=0}^{n-1}(1-\gamma_t)\le
\exp(-\sum_{t=0}^{n-1}\gamma_t)$ follows from standard survival-analysis
arguments.
\begin{theorem}\label{thm:hazard-bounds}
(i) Assume there exists a sequence of real numbers $\{a_t\}_{t\ge 1}\subset[0,1]$ such that for all $t\ge 1$,
$$
p_t \ge a_t \quad \text{on } H_{t-1}\ \text{almost surely}.
$$
Then for every $n\ge 1$,
$$
\Prob(\tau_A>n)\le \Prob(E_0^c)\prod_{t=1}^n (1-a_t).
$$
In particular, if $\sum_{t=1}^\infty a_t=\infty$, then $\Prob(\tau_A<\infty)=1$.
\smallskip

(ii) Assume there exists a sequence of real numbers $\{b_t\}_{t\ge 1}\subset[0,1]$ such that for all $t\ge 1$,
$$
p_t \le b_t \quad \text{on } H_{t-1}\ \text{almost surely}.
$$
Then for every $n\ge 1$,
$$
\Prob(\tau_A>n)\ge \Prob(E_0^c)\prod_{t=1}^n (1-b_t).
$$
Moreover, if $b_t\in[0,1)$ for all $t$ and $\sum_{t=1}^\infty b_t<\infty$, then
$$
\Prob(\tau_A=\infty)
=\lim_{n\to\infty}\Prob(\tau_A>n)
\;\ge\;
\Prob(E_0^c)\prod_{t=1}^\infty (1-b_t)
\;>\;0.
$$
\end{theorem}
\begin{corollary}\label{cor:explicit-tails}
If $a_t\equiv a\in(0,1]$, then
$$
\Prob(\tau_A>n)\le \Prob(E_0^c)(1-a)^n,
\qquad
\E[\tau_A]\le \sum_{n=0}^\infty \Prob(\tau_A>n)\le \frac{\Prob(E_0^c)}{a}.
$$
In particular, $\E[\tau_A]\le 1/a$ since $\Prob(E_0^c)\le 1$.
\smallskip

If $a_t = C t^{-\alpha}$ with $C$ and $\alpha \in (0,1)$, then for all $n\ge 1$,
$$
\Prob(\tau_A>n)
\le \Prob(E_0^c)\exp\!\left(-\frac{C}{1-\alpha}\bigl((n+1)^{1-\alpha}-1\bigr)\right)
= \Prob(E_0^c)\exp\!\left(\frac{C}{1-\alpha}\right)\exp\!\left(-\frac{C}{1-\alpha}(n+1)^{1-\alpha}\right).
$$

If $a_t = C t^{-1}$ with $C > 0$, then for all $n\ge 1$,
$$
\Prob(\tau_A>n)\le \Prob(E_0^c)(n+1)^{-C}.
$$
\end{corollary}

\begin{lemma}\label{lem:conditional-witness}
Let $\{L_t\}_{t\ge1}$ be events with $L_t\in\mathcal F_{t-1}$.
Assume that for all $t\ge1$,
$$
p_t \ge a_t \quad \text{on } H_{t-1}\cap L_t,
$$
where $\{a_t\}$ is deterministic sequence of numbers.
Then
$$
h_t = \Prob(E_t\mid H_{t-1}) \ge a_t \Prob(L_t\mid H_{t-1}).
$$
\end{lemma}

\begin{proof}
Since $H_{t-1}\in\mathcal F_{t-1}$, we have
$$
h_t
=
\Prob(E_t\mid H_{t-1})
=
\E\!\left[p_t \mid H_{t-1}\right].
$$
On the event $H_{t-1}\cap L_t$, the assumption gives $p_t\ge a_t$, while on
$H_{t-1}\cap L_t^{c}$ we only use the trivial bound $p_t\ge 0$.
Therefore,
$$
\E\!\left[p_t \mid H_{t-1}\right]
\ge
\E\!\left[a_t\,\mathbf 1_{L_t} \mid H_{t-1}\right]
=
a_t\,\Prob(L_t\mid H_{t-1}),
$$
which proves the claim.
\end{proof}
This decomposition separates algorithmic constants $\{a_t\}$ from empirical frequencies
$\Prob(L_t|H_{t-1})$. The remainder of the paper constructs concrete $L_t$ for L-SHADE.
\section{An L-SHADE Event for Current-to-$\boldsymbol p\,$best/1}\label{sec:lshade-event}

Throughout this section we focus on L-SHADE with \emph{current-to-$p$best/1} mutation 
and binomial crossover, targeting the sublevel set
$$
A_\varepsilon = \{x \in [l,u]^d : f(x) \le f^\star + \varepsilon\}.
$$
We work at a fixed generation $t \ge 1$. To simplify notation, we suppress 
time indices and write $x_i$, $N$, $P$, $\mathcal{A}$, etc.\ instead of 
$x_i^{(t-1)}$, $N^{(t-1)}$, $P^{(t-1)}$, $\mathcal{A}^{(t-1)}$ for the state 
at the \emph{start} of generation $t$ (equivalently, after generation $t-1$).
Similarly, $v_i$, $u_i$ denote the mutant and trial vectors $v_i^{(t)}$, $u_i^{(t)}$ 
produced during generation $t$.
The survival event is $H_{t-1} = \{\tau_{A_\varepsilon} > t-1\}$.

\subsection{Notation and Setup}\label{sec:notation}

\paragraph{Population and archive.}
Let $N$ denote the current population size.
The population is
$$
P = \{x_1, \ldots, x_N\} \subset [l,u]^d,
$$
with index set $\mathcal{P} = \{1, \ldots, N\}$.
Let $\mathcal{A}$ denote the \emph{archive} index set, disjoint from $\mathcal{P}$, 
with corresponding vectors $\{x_j : j \in \mathcal{A}\}$.
The archive stores parent vectors that were replaced by better offspring 
in previous generations; its maximum size is $\lfloor \mathrm{arc\_rate} \times N_{\mathrm{init}} \rfloor$
(default $\mathrm{arc\_rate} = 2.6$).
See \cite{Tanabe2014LSHADE} for the full L-SHADE specification.

\paragraph{The $p$-best set.}
Let $\mathcal{B} \subseteq \mathcal{P}$ denote the indices of the 
$m = \max\{1, \lceil p   N \rceil\}$ individuals with smallest objective values, 
where $p \in (0,1]$ is a fixed parameter (default $p = 0.11$).

\paragraph{Donor pools.}
For a target index $i \in \mathcal{P}$ and a chosen $p$-best index $b \in \mathcal{B}$, 
the admissible donor sets are:
\begin{align*}
S_i^{(1)} &= \mathcal{P} \setminus \{i, b\}, 
    &s_i^{(1)} &= |S_i^{(1)}| = N - 2, \\
S_i^{(2)} &= (\mathcal{P} \cup \mathcal{A}) \setminus \{i, b\}, 
    &s_i^{(2)} &= |S_i^{(2)}| = N + |\mathcal{A}| - 2.
\end{align*}
The first donor $r_1$ is drawn from $S_i^{(1)}$ (population only), 
while the second donor $r_2$ is drawn from $S_i^{(2)}$ (population $\cup$ archive).

\paragraph{Historical memory.}
L-SHADE maintains $H$ memory slots (default $H = 6$), storing 
location parameters $(M_F[k], M_{CR}[k])$ for $k \in \{1, \ldots, H\}$, 
all initialized to $0.5$.
When generating parameters for individual $i$, a memory slot $K_i$ is 
sampled uniformly from $\{1, \ldots, H\}$, and then:
\begin{itemize}
\item $F_i \sim \mathrm{Cauchy}(M_F[K_i], 0.1)$, regenerated if $\le 0$, truncated to $1$ if $> 1$;
\item $CR_i \sim \mathcal{N}(M_{CR}[K_i], 0.1)$, clipped to $[0,1]$.
\end{itemize}

\paragraph{Mutation (current-to-$p$best/1).}
Given target index $i$, $p$-best index $b$, donor indices $(r_1, r_2)$, 
and scaling factor $F \in (0,1]$, the mutant vector before boundary handling is
\begin{equation}\label{eq:mutant-def}
\tilde{v}_i(F; b, r_1, r_2) = x_i + F   \Delta_i(b, r_1, r_2),
\end{equation}
where the \emph{difference vector} is
$$
\Delta_i(b, r_1, r_2) = (x_b - x_i) + (x_{r_1} - x_{r_2}).
$$
Thus, for fixed indices, the mutant traces an affine ray $F \mapsto x_i + F   \Delta_i$ 
emanating from the parent $x_i$.

\paragraph{Boundary handling (midpoint repair).}
Let $\mathsf{BH}: \mathbb{R}^d \times [l,u]^d \to [l,u]^d$ denote the midpoint repair operator.
For a candidate $\tilde{v}$ and parent $x$, each coordinate is repaired as:
\begin{gather}\label{eq:boundary_handling}
\mathsf{BH}(\tilde{v}, x)_j =
\begin{cases}
(l_j + x_j)/2, & \text{if } \tilde{v}_j < l_j, \\
(u_j + x_j)/2, & \text{if } \tilde{v}_j > u_j, \\
\tilde{v}_j, & \text{otherwise}.
\end{cases}
\end{gather}
The \emph{repaired mutant} is
\begin{equation}\label{eq:repaired-mutant}
v_i(F; b, r_1, r_2) = \mathsf{BH}\bigl(\tilde{v}_i(F; b, r_1, r_2),\, x_i\bigr) \in [l,u]^d.
\end{equation}
The operator $\mathsf{BH}$ is Borel measurable (it is piecewise affine in each coordinate).

\paragraph{Binomial crossover.}
Given the repaired mutant $v_i$, parent $x_i$, crossover rate $CR_i \in [0,1]$, 
and a forced index $J_i \sim \mathrm{Unif}\{1,\ldots,d\}$, 
the \emph{trial vector} $u_i$ is defined coordinatewise by:
$$
(u_i)_j =
\begin{cases}
(v_i)_j, & \text{if } U_j \le CR_i \text{ or } j = J_i, \\
(x_i)_j, & \text{otherwise},
\end{cases}
$$
where $U_1, \ldots, U_d \stackrel{\mathrm{iid}}{\sim} \mathrm{Unif}(0,1)$.
The forced index $J_i$ ensures that at least one coordinate comes from the mutant.

\paragraph{Selection.}
If $f(u_i) \le f(x_i)$, the trial replaces the parent in the next generation; 
otherwise the parent survives. 
When a trial strictly improves ($f(u_i) < f(x_i)$), the replaced parent is 
added to the archive $\mathcal{A}$.

\paragraph{Population size reduction (LPSR).}
After selection, the population size is reduced according to
$$
N_{\mathrm{new}} = \mathrm{round}\!\left(
N_{\mathrm{init}} + (N_{\min} - N_{\mathrm{init}})   \frac{\mathrm{NFE}}{\mathrm{MaxNFE}}
\right),
$$
where $N_{\min} = 4$. The worst individuals are removed if $N_{\mathrm{new}} < N$.

\begin{lemma}\label{lem:trials-sublevel}
On the survival event $H_{t-1}$, we have
\begin{equation}\label{eq:Et-via-trials-new}
E_t=H_{t-1}\cap\left\{\exists i\in\{1,\dots,N_{t-1}\}:\ u^{(t)}_i\in A_\varepsilon\right\}.
\end{equation}
\end{lemma}

\begin{proof}
By definition of survival, for every parent in the current population
we have $x_i^{(t-1)}\notin A_\varepsilon$, hence $f(x_i^{(t-1)})>f^\star+\varepsilon$ for all $i$.
Now fix any index $i$ and suppose the trial lands in the success set: $u_{t,i}\in A_\varepsilon$.
Then $f(u_{t,i})\le f^\star+\varepsilon < f(x_i^{(t-1)})$, so selection \emph{must} accept the trial,
i.e.\ $x_i^{(t)}=u_{t,i}\in A_\varepsilon$.

Consequently, on $H_{t-1}$ the event that some trial enters $A_\varepsilon$
is equivalent to the event that the next population contains an individual in $A_\varepsilon$,
which is exactly the first-hit event $E_t$ (since $H_{t-1}$ implies there were no earlier hits).
This proves \eqref{eq:Et-via-trials-new}.
\end{proof}

\begin{remark}
If the implementation applies a {\it population reduction} step after selection (as in L-SHADE's LPSR),
then \eqref{eq:Et-via-trials-new} still holds provided the reduction removes only the \emph{worst}
individuals (largest objective values). Indeed, any accepted $u_{t,i}\in A_\varepsilon$ has
$f(u_{t,i})\le f^\star+\varepsilon$, hence is among the best and cannot be removed by remove worst.
\end{remark}

The preceding description specifies the L-SHADE operators deterministically 
given the random inputs. We now state the probabilistic sampling rules as formal assumptions.
\begin{assumption}[Index sampling]\label{ass:index-sampling}
Fix a target index $i \in \mathcal{P}$. Conditioned on $\mathcal{F}_{t-1}$:
\begin{enumerate}
\item The $p$-best index $b$ is sampled uniformly from $\mathcal{B}$, so 
    $\Prob(b = b^\star \mid \mathcal{F}_{t-1}) = 1/m$ for each $b^\star \in \mathcal{B}$.
    
\item The donor indices $(r_1, r_2)$ are sampled via \emph{two-pool selection}:
    $$
    r_1 \sim \mathrm{Unif}(S_i^{(1)}), \qquad 
    r_2 \mid r_1 \sim \mathrm{Unif}(S_i^{(2)} \setminus \{r_1\}).
    $$
    Since $S_i^{(1)} \subseteq S_i^{(2)}$, for any admissible pair $(r_1^\star, r_2^\star)$ 
    with $r_1^\star \in S_i^{(1)}$ and $r_2^\star \in S_i^{(2)} \setminus \{r_1^\star\}$,
    $$
    \Prob\bigl((r_1, r_2) = (r_1^\star, r_2^\star) \mid \mathcal{F}_{t-1}\bigr) 
    = \frac{1}{s_i^{(1)}(s_i^{(2)} - 1)}.
    $$
\end{enumerate}
\end{assumption}
Let us mark: The probability in Assumption \ref{ass:index-sampling}(2) is very small (e.g., $\approx 10^{-6}$ for $d=10$ with full archive).
In Section \ref{sec:morse-bounds} we show that the population concentration near the optimum yields a much tighter bound for Morse functions.

\subsection{An L-SHADE event and hazard lower bound}

We now define a measurable event (depending only on the state at time $t-1$) under which
we will find  a deterministic lower bound on the conditional first-hit probability
$p_t=\Prob(E_t\mid\mathcal F_{t-1})$.

Fix once and for all
$$
0<F^{-}<F^{+}\le 1,\qquad
\Delta_F>0,\qquad
c_{\mathrm{cr}}\in(0,1),\qquad
g^{-}>0,\qquad
q^{-}\in(0,1].
$$
Here $[F^{-}, F^{+}] \subset (0,1]$ is the working range for $F$ on which we 
require a density lower bound;
$\Delta_F > 0$ is the minimal Lebesgue measure of the success-$F$ window;
$c_{\mathrm{cr}} \in (0,1)$ is a crossover rate threshold (we restrict attention to non-degenerate crossover events with $CR_i \in [c_{\mathrm{cr}},1]$);
$g^{-} > 0$ is a lower bound on the $F$-density; and
$q^{-} \in (0,1]$ is a uniform lower bound on the probability $\Prob(CR_i \ge c_{\mathrm{cr}})$.

For each $(i,b,r_1,r_2)$ define the set of \emph{successful} scaling factors
\begin{equation}\label{eq:success-F-window}
\mathcal I_t(i,b,r_1,r_2) \;=\; \Big\{F\in[F^{-},F^{+}]:\ v^{(t)}_i(F;b,r_1,r_2)\in A_{\varepsilon/2}\Big\}.
\end{equation}
We emphasize that $\mathcal I_t(i,b,r_1,r_2)$ is a subset of $[F^{-},F^{+}]$ and can be a union of
intervals; $\lambda(\mathcal I_t)$ denotes its 1D Lebesgue measure (total length).


We encode this in the following assumptions.

\begin{assumption}[Parameter sampling]\label{ass:slot-sampling}
Fix a target index $i \in \mathcal{P}$. Conditioned on $\mathcal{F}_{t-1}$:
\begin{enumerate}
\item The memory slot $K_i$ is sampled uniformly from $\{1, \ldots, H\}$, so
    $\Prob(K_i = k \mid \mathcal{F}_{t-1}) = 1/H$ for each $k$.
\item Conditional on $(\mathcal{F}_{t-1}, K_i = k)$, the random variables $F_i$ and $CR_i$
    are independent, with $F_i$ admitting a density $g^F_k( )$ on $(0,1]$.
\end{enumerate}
\end{assumption}

\begin{remark}\label{rem:F_CR_independence}
Assumption~\ref{ass:slot-sampling} matches the L-SHADE implementation: 
pick memory index $k$, then draw $F$ from a truncated Cauchy centered at $M_F[k]$ 
and $CR$ from a clipped normal centered at $M_{CR}[k]$. The numbers $F$ and $CR$ are conditionally independent given the index $k$.
\end{remark}

\begin{assumption}[Binomial crossover]\label{ass:crossover}
Conditioned on $(\mathcal{F}_{t-1}, CR_i)$, the trial $u_i$ is produced by:
\begin{enumerate}
\item sample $J_i$ uniformly from $\{1, \ldots, d\}$;
\item for each $j \neq J_i$, inherit coordinate $j$ from the mutant $v_i$ with probability $CR_i$
    (independently across coordinates);
\item inherit coordinate $J_i$ from the mutant deterministically.
\end{enumerate}
\end{assumption}

\begin{remark}
In L-SHADE with binomial crossover, it may occur that the trial vector
$u_{i}$ coincides with the mutant vector $v_{i}$ except for the
forced coordinate $j_{\mathrm{rand}}$, so that the sampled crossover
rate $CR_{i}$ does not effectively influence the offspring.
In such cases, the original L-SHADE algorithm records
$CR_{i}=\perp$ and excludes this value from the update of the
crossover-rate memory.
Throughout this paper, whenever we condition on $CR_{i}$ or on events
involving crossover, this conditioning is understood to be restricted
to the event of non-degenerate crossover, on which $CR_{i}$ is
well-defined.
\end{remark}

We can now define the event that there exists a \emph{witness configuration} in the current state.

\begin{definition}[L-SHADE witness event $\mathcal{L}_t$]
\label{def:Lshade-event}
Let $\mathcal{L}_t \in \mathcal{F}_{t-1}$ be the event that there exist:
a target index $i \in \mathcal{P}$,
a $p$-best index $b \in \mathcal{B}$,
donor indices $r_1 \in S_i^{(1)}$ and $r_2 \in S_i^{(2)} \setminus \{r_1\}$,
and a memory slot $k \in \{1, \ldots, H\}$ such that:
\begin{align}
\lambda\!\bigl(\mathcal{I}(i, b, r_1, r_2)\bigr) &\ge \Delta_F,
\label{eq:L1}\\
\inf_{F \in [F^{-}, F^{+}]} g^F_k(F) &\ge g^{-},
\label{eq:L2}\\
\Prob(CR_i \ge c_{\mathrm{cr}} \mid \mathcal{F}_{t-1}, K_i = k) &\ge q^{-}.
\label{eq:L3}
\end{align}
\end{definition}

\begin{remark}
By construction, $\mathcal{L}_t$ depends only on the state at time $t-1$:
the population and archive determine the index sets and the difference vector $\Delta_i(b, r_1, r_2)$,
while the L-SHADE memories determine the density $g^F_k$ and the distribution of $CR$.
Thus $\mathcal{L}_t \in \mathcal{F}_{t-1}$.

However, condition \eqref{eq:L1} involves the set of $F$ values for which
$f(v_i(F)) \le f^\star + \varepsilon/2$, which depends on the objective geometry.
In general black-box settings this is \emph{not} directly observable without function knowledge;
one may (i) verify it analytically for structured $f$, or 
(ii) approximate it numerically by evaluating $f(v_i(F))$ on a grid in $F$.
\end{remark}

\subsection{A lower bound on $p_t$ on $H_{t-1} \cap \mathcal{L}_t$}

For an integer $r \in \{0, \ldots, d-1\}$, define the event
\begin{equation}\label{eq:C_r-def}
C_r(i) = \{\text{the trial } u_i \text{ inherits at least } d-r \text{ coordinates from } v_i\}.
\end{equation}
Equivalently, $C_r(i)$ means ``at most $r$ coordinates come from the parent $x_i$''.

\begin{lemma}\label{lem:eta-def}
Under Assumption~\ref{ass:crossover}, conditional on $CR_i = c$, the number of mutant 
coordinates inherited into $u_i$ is
$$
1 + \mathrm{Bin}(d-1, c),
$$
so for any $r \in \{0, \ldots, d-1\}$,
$$
\Prob\bigl(C_r(i) \mid CR_i = c\bigr)
= \Prob\!\left(\mathrm{Bin}(d-1, c) \ge d - r - 1\right).
$$
In particular, define the deterministic constant
\begin{equation}\label{eq:eta}
\eta_r(d, c_{\mathrm{cr}}) = \Prob\!\left(\mathrm{Bin}(d-1, c_{\mathrm{cr}}) \ge d - r - 1\right).
\end{equation}
Then on the event $\{CR_i \ge c_{\mathrm{cr}}\}$,
\begin{equation}\label{eq:eta-lower}
\Prob\bigl(C_r(i) \mid \mathcal{F}_{t-1}, CR_i\bigr) \ge \eta_r(d, c_{\mathrm{cr}}).
\end{equation}
\end{lemma}
\begin{proof}
Assumption~\ref{ass:crossover} says: one coordinate is forced from the mutant, and each of the
remaining $d-1$ coordinates is taken from the mutant independently with probability $c$.
Thus the number of mutant coordinates is $1 + \mathrm{Bin}(d-1, c)$.
The event $C_r(i)$ is exactly the event that this number is at least $d-r$,
i.e., $\mathrm{Bin}(d-1, c) \ge d - r - 1$. This gives the first identity.

For the monotonicity bound \eqref{eq:eta-lower}, note that the map
$c \mapsto \Prob(\mathrm{Bin}(d-1, c) \ge d - r - 1)$ is nondecreasing in $c \in [0,1]$,
so on $\{CR_i \ge c_{\mathrm{cr}}\}$, the conditional probability of $C_r(i)$ is at least
$\eta_r(d, c_{\mathrm{cr}})$.
\end{proof}

\begin{remark}\label{rem:eta-bound}
For any $c_{\mathrm{cr}} \in (0,1)$, choosing $r = d - 1 - \lfloor(d-1)c_{\mathrm{cr}}\rfloor$ yields
$$
\eta_r(d, c_{\mathrm{cr}}) \geq \frac{1}{2} \quad \text{uniformly in } d,
$$
since the median of $\mathrm{Bin}(n,p)$ lies in $\{\lfloor np \rfloor, \lceil np \rceil\}$.
For the L-SHADE initialization $c_{\mathrm{cr}} = 1/2$, this gives $r = \lfloor(d-1)/2\rfloor$.
\end{remark}

To use $C_r(i)$, we need an implication of the form (see Lemma~\ref{lem:eta-def} and \eqref{eq:L1}):
$$
v_i\in A_{\varepsilon/2}\ \text{and}\ \|x_i-v_i\|_\infty\le\delta\ \text{and}\ C_r(i)\ \Longrightarrow\ u_i\in A_\varepsilon.
$$

\begin{definition}\label{def:crossover-stable}
Let us fix $\varepsilon > 0$ and an integer $r \in \{0, \dots, d-1\}$. 
We say that $f$ is {\it $(\varepsilon, r, \delta)$-crossover-stable} on $[l,u]^d$ if the following holds: for every $v \in A_{\varepsilon/2}$, every $x \in [l,u]^d$ with $\|x - v\|_\infty \leq \delta$, and every index set $J \subseteq \{1, \dots, d\}$ with $|J| \ge d - r$, the vector $u$ defined by
$$
u_j = \begin{cases} v_j & \text{if } j \in J, \\ x_j & \text{if } j \notin J, \end{cases}
$$
satisfies $u \in A_\varepsilon$.
\end{definition}

Proposition~\ref{prop:Cr-sufficient-new} establishes that every Morse function is $(\varepsilon, r, \delta)$-crossover-stable with appropriate sufficiently small $\varepsilon$ and $\delta\leq\delta_{\max}(\varepsilon,r)$.

The decomposition of $\gamma_{t,i}$ into conditional sub-events follows standard probabilistic arguments for adapted processes; see, e.g., \citep{Kallenberg2021}.

\begin{lemma}[Hazard bound on $H_{t-1} \cap \mathcal{L}_t$]
\label{lem:lshade-pt}
Assume:
\begin{enumerate}
\item Assumptions~\ref{ass:index-sampling}--\ref{ass:crossover} hold.
\item There exists $r \in \{0, \ldots, d-1\}$ such that on the event $\mathcal{L}_t$,
for the witness target index $i$ and witness indices $(b, r_1, r_2)$ in
Definition~\ref{def:Lshade-event},
\begin{equation}\label{eq:stability-implication}
v_i(F; b, r_1, r_2) \in A_{\varepsilon/2}\ \text{and}\ \|x_i-v_i\|_\infty\le\delta \text{ and } C_r(i) \quad \Longrightarrow \quad u_i \in A_\varepsilon.
\end{equation}
\end{enumerate}
Then on $H_{t-1} \cap \mathcal{L}_t$,
\begin{equation}\label{eq:pt-lower-L}
p_t = \Prob(E_t \mid \mathcal{F}_{t-1}) \ge a_t,
\end{equation}
where
\begin{equation}\label{eq:a_t-L}
a_t = \frac{1}{H}   \frac{1}{m}   \frac{1}{s_i^{(1)}(s_i^{(2)} - 1)} 
        (g^{-} \Delta_F)   (q^{-} \eta_r(d, c_{\mathrm{cr}})).
\end{equation}
Here $i$ is the witness target index from $\mathcal{L}_t$.
\end{lemma}

\begin{proof}
Work on the event $H_{t-1} \cap \mathcal{L}_t$.
Since $\mathcal{L}_t \in \mathcal{F}_{t-1}$, conditioning on $\mathcal{F}_{t-1}$ treats the witness objects
($i$, and the sets $\mathcal{B}$, $S_i^{(1)}$, $S_i^{(2)}$, the window $\mathcal{I}( )$,
and the slot-dependent distributions) as fixed.
\smallskip

{\sl Step 1: Reduce $E_t$ to a single individual's trial.}
By Lemma~\ref{lem:trials-sublevel}, on $H_{t-1}$ we have
$$
E_t = H_{t-1} \cap \{\exists j : u_j \in A_\varepsilon\}.
$$
In particular,
$$
\{u_i \in A_\varepsilon\} \subseteq \{\exists j : u_j \in A_\varepsilon\},
$$
so on $H_{t-1}$ we have $\{u_i \in A_\varepsilon\} \subseteq E_t$.
Taking conditional probabilities given $\mathcal{F}_{t-1}$ yields
\begin{equation}\label{eq:pt-ge-single}
p_t = \Prob(E_t \mid \mathcal{F}_{t-1}) \ge \Prob(u_i \in A_\varepsilon \mid \mathcal{F}_{t-1}).
\end{equation}
\smallskip

{\sl Step 2: Define a concrete success event for target $i$.}
Let $(b^\star, r_1^\star, r_2^\star, k^\star)$ be a witness quadruple guaranteed by $\mathcal{L}_t$,
so that \eqref{eq:L1}--\eqref{eq:L3} hold.
Consider the event $\mathcal{A}$ that, in the construction of the trial for target $i$,
the algorithm selects exactly these discrete choices:
$$
\mathcal{A} = \{b = b^\star\} \cap \{(r_1, r_2) = (r_1^\star, r_2^\star)\} \cap \{K_i = k^\star\}.
$$
By Assumptions~\ref{ass:index-sampling} and \ref{ass:slot-sampling},
\begin{equation}\label{eq:A-prob}
\Prob(\mathcal{A} \mid \mathcal{F}_{t-1}) = \frac{1}{m}   \frac{1}{s_i^{(1)}(s_i^{(2)} - 1)}   \frac{1}{H}.
\end{equation}
\smallskip

{\sl Step 3: Bound the probability of hitting the success-$F$ window.}
On $\mathcal{A}$, the mutant has the fixed affine form $F \mapsto v_i(F; b^\star, r_1^\star, r_2^\star)$.
By the witness condition \eqref{eq:L1}, the set $\mathcal{I}(i, b^\star, r_1^\star, r_2^\star)$ has
Lebesgue measure at least $\Delta_F$.
By \eqref{eq:L2}, the $F$-density under slot $k^\star$ satisfies
$g^F_{k^\star}(F) \ge g^{-}$ for all $F \in [F^{-}, F^{+}]$.
Therefore,
\begin{align}
\Prob\bigl(F_i \in \mathcal{I}(i, b^\star, r_1^\star, r_2^\star) \mid \mathcal{F}_{t-1}, \mathcal{A}\bigr)
&= \int_{\mathcal{I}(i, b^\star, r_1^\star, r_2^\star)} g^F_{k^\star}(F)\, dF
\nonumber\\
&\ge \int_{\mathcal{I}(i, b^\star, r_1^\star, r_2^\star)} g^{-}\, dF
= g^{-} \lambda\bigl(\mathcal{I}(i, b^\star, r_1^\star, r_2^\star)\bigr)
\nonumber\\
&\ge g^{-} \Delta_F.
\label{eq:F-hit-prob}
\end{align}
\smallskip

{\sl Step 4: Bound the crossover probability.}
By \eqref{eq:L3},
$$
\Prob(CR_i \ge c_{\mathrm{cr}} \mid \mathcal{F}_{t-1}, \mathcal{A}) \ge q^{-}.
$$
Condition on $(\mathcal{F}_{t-1}, \mathcal{A}, CR_i)$ and apply Lemma~\ref{lem:eta-def}:
on $\{CR_i \ge c_{\mathrm{cr}}\}$, the conditional probability of $C_r(i)$ is at least
$\eta_r(d, c_{\mathrm{cr}})$.
Thus,
\begin{align}
\Prob\bigl(C_r(i) \cap \{CR_i \ge c_{\mathrm{cr}}\} \mid \mathcal{F}_{t-1}, \mathcal{A}\bigr)
&= \E\!\left[\mathbf{1}_{\{CR_i \ge c_{\mathrm{cr}}\}}
   \Prob(C_r(i) \mid \mathcal{F}_{t-1}, \mathcal{A}, CR_i)
   \;\middle|\; \mathcal{F}_{t-1}, \mathcal{A}\right]
\nonumber\\
&\ge \E\!\left[\mathbf{1}_{\{CR_i \ge c_{\mathrm{cr}}\}} \eta_r(d, c_{\mathrm{cr}})
   \;\middle|\; \mathcal{F}_{t-1}, \mathcal{A}\right]
\nonumber\\
&= \eta_r(d, c_{\mathrm{cr}}) \Prob(CR_i \ge c_{\mathrm{cr}} \mid \mathcal{F}_{t-1}, \mathcal{A})
\nonumber\\
&\ge q^{-} \eta_r(d, c_{\mathrm{cr}}).
\label{eq:Cr-prob}
\end{align}
\smallskip

{\sl Step 5: Combine the bounds.}
Consider the event
$$
\mathcal{S} = \mathcal{A} \cap \{F_i \in \mathcal{I}(i, b^\star, r_1^\star, r_2^\star)\}
              \cap \{CR_i \ge c_{\mathrm{cr}}\} \cap C_r(i).
$$
On $\mathcal{S}$, by the definition of $\mathcal{I}$ we have
$v_i(F_i; b^\star, r_1^\star, r_2^\star) \in A_{\varepsilon/2}$.
Together with $C_r(i)$, the stability implication \eqref{eq:stability-implication} gives
$u_i \in A_\varepsilon$.
Hence $\mathcal{S} \subseteq \{u_i \in A_\varepsilon\}$ and therefore
$$
\Prob(u_i \in A_\varepsilon \mid \mathcal{F}_{t-1}) \ge \Prob(\mathcal{S} \mid \mathcal{F}_{t-1}).
$$

Using the chain rule (conditioning on $\mathcal{A}$) and Assumption~\ref{ass:slot-sampling}
(conditional independence of $F_i$ and $CR_i$ given the slot),
we multiply the bounds \eqref{eq:A-prob}, \eqref{eq:F-hit-prob}, and \eqref{eq:Cr-prob}:
\begin{align*}
\Prob(\mathcal{S} \mid \mathcal{F}_{t-1})
&= \Prob(\mathcal{A} \mid \mathcal{F}_{t-1})
     \Prob\bigl(F_i \in \mathcal{I}( ) \mid \mathcal{F}_{t-1}, \mathcal{A}\bigr)
     \Prob\bigl(C_r(i) \cap \{CR_i \ge c_{\mathrm{cr}}\} \mid \mathcal{F}_{t-1}, \mathcal{A}\bigr)
\\
&\ge \frac{1}{H}   \frac{1}{m}   \frac{1}{s_i^{(1)}(s_i^{(2)} - 1)}
       (g^{-} \Delta_F)   (q^{-} \eta_r(d, c_{\mathrm{cr}})).
\end{align*}
Combining with \eqref{eq:pt-ge-single} yields \eqref{eq:pt-lower-L} with $a_t$ as in \eqref{eq:a_t-L}.
\end{proof}

\begin{remark}[Structure of the bound \eqref{eq:a_t-L}]\label{rem:a_t-structure}
Fix the thresholds $(F^{-}, F^{+}, \Delta_F, c_{\mathrm{cr}}, g^{-}, q^{-})$ and the integer $r$.
The lower bound \eqref{eq:a_t-L} decomposes into three types of factors.

\begin{enumerate}
\item \textbf{Algorithmic/combinatorial factors (computable).}
From the discrete sampling rules:
$$
\frac{1}{H}, \qquad \frac{1}{m}, \qquad \frac{1}{s_i^{(1)}(s_i^{(2)} - 1)}.
$$
For L-SHADE defaults ($H = 6$, $p = 0.11$) at the initial population size $N = N_{\mathrm{init}} = 18d$,
we have $m = \max\{1, \lceil pN \rceil\}$ (e.g., for $d = 10$, $m = \lceil 0.11   180 \rceil = 20$,
so $1/m = 0.05$).

For the two-pool sampling (Assumption~\ref{ass:index-sampling}),
$$
s_i^{(1)} = N - 2, \qquad s_i^{(2)} = N + |\mathcal{A}| - 2.
$$
At archive capacity with $\mathrm{arc\_rate} = 2.6$ and $d = 10$ 
(so $N = 180$ and $|\mathcal{A}| \approx 2.6   N_{\mathrm{init}} = 468$), 
we have $s_i^{(1)} = 178$ and $s_i^{(2)} = 646$, hence
$$
\frac{1}{s_i^{(1)}(s_i^{(2)} - 1)} = \frac{1}{178   645} \approx 8.7 \times 10^{-6}.
$$
Consequently,
$$
\frac{1}{H}   \frac{1}{m}   \frac{1}{s_i^{(1)}(s_i^{(2)} - 1)}
\approx \frac{1}{6}   \frac{1}{20}   8.7 \times 10^{-6}
\approx 7.3 \times 10^{-8}.
$$
Note that due to population-size reduction (LPSR), $m$, $s_i^{(1)}$, and $s_i^{(2)}$ typically decrease over time.

The crossover factor $\eta_r(d, c_{\mathrm{cr}}) = \Prob(\mathrm{Bin}(d-1, c_{\mathrm{cr}}) \ge d - r - 1)$
is a binomial tail probability that can be evaluated numerically.

\item \textbf{State-checkable factors.}
The witness conditions \eqref{eq:L2}--\eqref{eq:L3} depend only on the L-SHADE memory state, so
$$
\inf_{F \in [F^{-}, F^{+}]} g^F_k(F), \qquad
\Prob(CR_i \ge c_{\mathrm{cr}} \mid \mathcal{F}_{t-1}, K_i = k)
$$
are checkable from the algorithm's internal state.

\item \textbf{Problem-dependent factors.}
The quantity $\lambda(\mathcal{I}(i, b, r_1, r_2))$ depends on the objective geometry 
and cannot be computed without function knowledge 
(though it may be bounded analytically for structured $f$).

\item \textbf{Empirical frequency.}
The survival-conditional frequency $\gamma_t = \Prob(\mathcal{L}_t \mid H_{t-1})$ 
reflects the algorithm--problem interaction and must be estimated empirically.
\end{enumerate}
\end{remark}

\subsection{Tail bounds via event frequency}

Define the deterministic survival-only hazard
$$
h_t = \Prob(E_t\mid H_{t-1}),\qquad t\ge 1.
$$
Then the survival events satisfy the exact recursion
$$
\Prob(H_t)
=
\Prob(H_{t-1})\Prob(H_t\mid H_{t-1})
=
\Prob(H_{t-1})\bigl(1-h_t\bigr),
$$
so by iterating,
\begin{equation}\label{eq:exact-survival-product}
\Prob(\tau_{A_\varepsilon}>n)
=
\Prob(E_0^c)\prod_{t=1}^n(1-h_t).
\end{equation}

Now define the \emph{L-SHADE event frequency on survival} by
$$
\gamma_t = \Prob(\mathcal L_t\mid H_{t-1})\in[0,1].
$$

\begin{proposition}
\label{prop:tail-via-L}
Assume the hypotheses of Lemma~\ref{lem:lshade-pt}, and suppose $a_t$ in \eqref{eq:a_t-L} is deterministic.
Then for every $t\ge 1$,
$$
h_t=\Prob(E_t\mid H_{t-1})
=
\E[p_t\mid H_{t-1}]
\ \ge\
a_t\,\Prob(\mathcal L_t\mid H_{t-1})
=
a_t\,\gamma_t,
$$
and consequently, for all $n\ge 1$,
\begin{equation}\label{eq:tail-L}
\Prob(\tau_{A_\varepsilon}>n)
\le
\Prob(E_0^c)\prod_{t=1}^n(1-a_t\gamma_t)
\le
\Prob(E_0^c)\exp\!\left(-\sum_{t=1}^n a_t\gamma_t\right).
\end{equation}
In particular, if $\sum_{t\ge 1} a_t\gamma_t=\infty$, then $\Prob(\tau_{A_\varepsilon}<\infty)=1$.
\end{proposition}

\begin{proof}
By definition, $p_t=\Prob(E_t\mid\mathcal F_{t-1})$ and $H_{t-1}\in\mathcal F_{t-1}$, so
$$
h_t=\Prob(E_t\mid H_{t-1})
=
\E[\mathbf 1_{E_t}\mid H_{t-1}]
=
\E\!\left[\E[\mathbf 1_{E_t}\mid\mathcal F_{t-1}]\ \middle|\ H_{t-1}\right]
=
\E[p_t\mid H_{t-1}],
$$
which is just the tower property.

On the event $H_{t-1}\cap\mathcal L_t$, Lemma~\ref{lem:lshade-pt} gives $p_t\ge a_t$.
Therefore,
$$
\E[p_t\mid H_{t-1}]
\ge
\E[a_t\,\mathbf 1_{\mathcal L_t}\mid H_{t-1}]
=
a_t\,\Prob(\mathcal L_t\mid H_{t-1})
=
a_t\gamma_t,
$$
where we used that $a_t$ is deterministic.
Insert $h_t\ge a_t\gamma_t$ into the exact product identity \eqref{eq:exact-survival-product}
to get the product bound, and use $1-x\le e^{-x}$ to obtain the exponential bound.
The final claim follows because if $\sum_t a_t\gamma_t=\infty$ then
$\exp(-\sum_{t=1}^n a_t\gamma_t)\to 0$, so $\Prob(\tau_{A_\varepsilon}=\infty)=0$.
\end{proof}

\begin{remark}
The bound \eqref{eq:tail-L} decomposes the survival probability into:
\begin{itemize}
\item a \emph{theoretical factor} $a_t$ (explicit but typically small), and
\item an \emph{empirical frequency} $\gamma_t = \Prob(\mathcal{L}_t \mid H_{t-1})$ measuring how often the witness configuration occurs.
\end{itemize}
Even if $a_t$ is tiny, the cumulative sum $\sum_{t \le n} a_t \gamma_t$ can grow large 
when $\mathcal{L}_t$ occurs frequently, explaining why worst-case bounds are conservative 
yet almost-sure convergence still holds.
\end{remark}

\section{Morse functions}\label{sec:morse-bounds}

Let $D \subset \mathbb{R}^d$ be a bounded domain such that $[l,u]^d \subset D$, and let $f \in C^2(D,\mathbb{R})$. Recall that $f$ is a Morse function on $D$ if all its critical points are nondegenerate, i.e., $\nabla f(x) = 0$ yields $\det\left(\nabla^2 f(x)\right) \neq 0$, $x \in \operatorname{int}(D)$. 

It is well known that a Morse function defined on $[l,u]^d$ has finitely many critical points, and hence finitely many local minima.

Let $x^\star \in \text{\rm Int}\,([l,u]^d)$ be a local minimum of $f$.
By Taylor's theorem, there exists a neighborhood of $x^\star$ such that
\begin{equation}\label{eq:morse_lemma}
f(x^\star + y)
=
f(x^\star)
+ \frac{1}{2}\, y^\top H(x^\star)\, y
+ o(\|y\|^2),
\end{equation}
where $H(x^\star) = \nabla^2 f(x^\star)$ is the Hessian at $x^\star$.
Throughout, for symmetric matrices $A,B$, we write $A \preceq B$ to denote that $B-A$ is positive semidefinite.

Since $x^\star$ is a local minimum and $f$ is Morse, the Hessian $H(x^\star)$ is positive definite.
Let
$$
\mu_0 = \lambda_{\min}\!\bigl(H(x^\star)\bigr) > 0,
\qquad
L_0 = \lambda_{\max}\!\bigl(H(x^\star)\bigr) < \infty.
$$

\begin{remark}\label{rem:morse}
There exist constants $0 < \mu \le L < \infty$ and a radius $r_0>0$ such that
$$
\mu I \preceq \nabla^2 f(x) \preceq L I,
\qquad
\forall x \in B(x^\star,r_0).
$$
This follows from continuity of $\nabla^2 f$ and $H(x^\star)\succ0$; for example one may take $\mu=\mu_0/2$ and $L=2L_0$ for $r_0$ small enough.
\end{remark}

The ratio $\kappa = \frac{L}{\mu} \ge 1$ is the condition number of the Hessian in the neighborhood of the minimum $x^\star$.

The function $f$ has an $L$-Lipschitz continuous gradient and is $\mu$-strongly convex on $B(x^\star,r_0)$.
For $\varepsilon>0$ define the local sublevel set
$$
A_\varepsilon=\{x\in B(x^\star,r_0): f(x)\le f(x^\star)+\varepsilon\}.
$$
If $2\varepsilon\le \mu r_0^2$, then $A_\varepsilon$ is convex and satisfies 
\begin{equation}\label{eq:sublevel-ellipsoid}
B\!\left(x^\star,\sqrt{\tfrac{2\varepsilon}{L}}\right)
\ \subseteq\ 
A_\varepsilon
\ \subseteq\ 
B\!\left(x^\star,\sqrt{\tfrac{2\varepsilon}{\mu}}\right).
\end{equation}

Throughout this section, we assume that $2\varepsilon\leq \mu r^2_0$.
\begin{remark}[Replacing the Morse assumption by quadratic growth]
The arguments in Section~\ref{sec:morse-bounds} do not rely on global Morse theory
(nondegeneracy of saddles/maxima). It suffices that, for the relevant minimizer set $X^\star$,
there exist constants $\mu,L>0$ and a neighborhood $U$ of $X^\star$ such that for all $x\in U$,
$$
f^\star + \frac{\mu}{2}\,\mathrm{dist}(x,X^\star)^2
\ \le\ f(x)\ \le\
f^\star + \frac{L}{2}\,\mathrm{dist}(x,X^\star)^2.
$$
The lower bound is a uniform quadratic growth condition; for nonsmooth objectives it can be
formulated via tilt stability/metric regularity of the limiting subdifferential
\cite{DrusvyatskiyLewis2013,RockafellarWets1998}.
In the smooth case, these local bounds follow from the second-order sufficient condition
$\nabla f(x^\star)=0$ and $\nabla^2 f(x^\star)\succ 0$ together with continuity of $\nabla^2 f$
\cite{NocedalWright2006}, which is also guaranteed at a nondegenerate Morse minimum
\cite{Milnor1963}.
\end{remark}

\subsection{Some preliminary results}
\begin{lemma}\label{lem:safe-ball}
Let $f$ be a Morse function, and let the number $\varepsilon$ be chosen such that $A_{\varepsilon/2}\subset B(x^\star,r_0)$.
If $x\in A_{\varepsilon/4}$, then
\begin{equation}\label{eq:safe-ball}
B\left(x,\, r_{\mathrm{safe}}\right) \subseteq A_{\varepsilon/2},
\qquad
r_{\mathrm{safe}}
=\left(1-\frac{1}{\sqrt{2}}\right)\sqrt{\frac{\varepsilon}{L}}.
\end{equation}
\end{lemma}

\begin{proof}
Let $y\in B(x,r_{\mathrm{safe}})$ and set $t=\|y-x\|$.
Since $x\in A_{\varepsilon/4}\subset B(x^\star,r_0)$, the Hessian bound
$\nabla^2 f \preceq L I$ applies, and Taylor’s theorem yields
$$
f(y)
\le
f(x)+\langle\nabla f(x),y-x\rangle+\frac{L}{2}t^2.
$$

Moreover, since $f$ is $L$-smooth on $B(x^\star,r_0)$ with minimizer $x^\star$,
we have $\|\nabla f(x)\|^2 \le 2L\,(f(x)-f^\star)$ for all $x\in B(x^\star,r_0)$.
Thus, for $x\in A_{\varepsilon/4}$,
$$
\|\nabla f(x)\|\le \sqrt{2L  \frac{\varepsilon}{4}}=\sqrt{\frac{L\varepsilon}{2}}.
$$

Thus,
$$
f(y)-f(x)
\le
t\sqrt{\frac{L\varepsilon}{2}}+\frac{L}{2}t^2.
$$
For $t\le r_{\mathrm{safe}}$, the right-hand side is at most $\varepsilon/4$,
and therefore
$$
f(y)\le f(x)+\frac{\varepsilon}{4}\le f^\star+\frac{\varepsilon}{2}.\qedhere
$$
\end{proof}

\begin{proposition}\label{prop:Cr-sufficient-new}
Let $f$ be a Morse function, and let the number $\varepsilon$ be chosen such that $A_{\varepsilon/2}\subset B(x^\star,r_0)$.
Define $\delta_{\max}(\varepsilon,r)= (\sqrt{2}-1)\sqrt{\frac{\varepsilon}{Lr}}$.

Then $f$ is $(\varepsilon,r,\delta)$-crossover-stable (see Definition \ref{def:crossover-stable}) on $[l,u]^d$ for every $0\le \delta \le \delta_{\max}(\varepsilon,r)$.
\end{proposition}

\begin{proof}
Since $u$ differs from $v$ in at most $r$ coordinates with $|u_j - v_j| \le \delta$, 
we have $\|u - v\| \le \sqrt{r}\,\delta$.

For $v \in A_{\varepsilon/2} \subset B(x^\star, r_0)$, strong convexity gives 
$\|v - x^\star\| \le \sqrt{\varepsilon/\mu}$.
For $\delta \le \delta_{\max}(\varepsilon,r)$,
$$
\|u - x^\star\| \le \|u - v\| + \|v - x^\star\|
\le (\sqrt{2}-1)\sqrt{\frac{\varepsilon}{L}} + \sqrt{\frac{\varepsilon}{\mu}}
\le \sqrt{\frac{2\varepsilon}{\mu}} \le r_0,
$$
so $u \in B(x^\star, r_0)$.

Since $f$ is convex and $L$-smooth on $B(x^\star, r_0)$, the standard gradient bound $\|\nabla f(v)\|^2 \le 2L(f(v) - f^\star)$ holds, giving  $\|\nabla f(v)\| \le \sqrt{L\varepsilon}$.
Taylor expansion yields
$$
f(u) \le f(v) + \sqrt{L\varepsilon} \sqrt{r}\,\delta + \frac{L}{2}r\delta^2.
$$
Since $f(v) \le f^\star + \varepsilon/2$, the condition $f(u) \le f^\star + \varepsilon$ reduces to
$$
\sqrt{L\varepsilon}  \sqrt{r}\,\delta + \frac{Lr\delta^2}{2} \le \frac{\varepsilon}{2}.
$$
Substituting $t = \sqrt{Lr/\varepsilon}\,\delta$ gives $t + t^2/2 \le 1/2$, 
which holds for $t \le \sqrt{2} - 1$.
\end{proof}

Before stating the next result, we isolate a special geometric case that
captures the core mechanism behind successful donor steps.
We first analyze the situation $i=b$, for which the current-to-$p$best/1
mutation reduces to a pure donor ray emanating from $x_b$.

Proposition~\ref{prop:success-F-ib} is only relevant once the run has produced a $p$-best donor in the deeper sublevel set $A_{\varepsilon/4}$.
Empirically, this is not a rare event in the stabilized exploitation phase: in Supplementary Material~SM--II.2 we show (CEC2017, $D=10$ and $D=30$) that
when the best first enters $A_{\varepsilon/4}$, a large fraction of the population
already lies in the surrounding basin sublevel set.
Hence the $p$-best pool
contains the deep point and the probability of selecting it as $x_b$ is bounded away from zero for the population sizes encountered by L-SHADE.

\begin{proposition}\label{prop:success-F-ib}
Let $f$ be a Morse function and assume $\varepsilon>0$ is such that
$A_{\varepsilon/2}\subset B(x^\star,r_0)$.
Fix indices $b\in\mathcal P$ and $(r_1,r_2)\in S_{t,b}^{(1)}\times S_{t,b}^{(2)}$ with $r_2\neq r_1$,
and suppose $x_b\in A_{\varepsilon/4}$.
Define the (boundary unrepaired) mutant ray
$$
\tilde{v}_b(F)=x_b+F\,(x_{r_1}-x_{r_2}),\qquad F\in[0,1].
$$
Fix any $0<F^-<F^+\le 1$ and any $\Delta_F\in(0,F^+-F^-]$.
If
\begin{equation}\label{eq:Delta-bound}
\|x_{r_1}-x_{r_2}\| \le \frac{r_{\mathrm{safe}}}{F^-+\Delta_F},
\end{equation}
then
$$
[F^-,\,F^-+\Delta_F]\subseteq \mathcal I_t(b,b,r_1,r_2),\quad  \text{i.e.}\quad 
\lambda\!\bigl(\mathcal I_t(b,b,r_1,r_2)\bigr)\ge \Delta_F.
$$
\end{proposition}

\begin{proof}
Let $F\in[F^-,F^-+\Delta_F]$. Then
$$
\|\tilde{v}_b(F)-x_b\|
=
F\,\|x_{r_1}-x_{r_2}\|
\le
(F^-+\Delta_F)\,\frac{r_{\mathrm{safe}}}{F^-+\Delta_F}
=
r_{\mathrm{safe}}.
$$
Hence $\tilde{v}_b(F)\in B(x_b,r_{\mathrm{safe}})$.
Since $x_b\in A_{\varepsilon/4}$, Lemma~\ref{lem:safe-ball} yields
$B(x_b,r_{\mathrm{safe}})\subseteq A_{\varepsilon/2}$, so $\tilde{v}_b(F)\in A_{\varepsilon/2}$.
Therefore $F\in\mathcal I_t(b,b,r_1,r_2)$ for every $F\in[F^-,F^-+\Delta_F]$,
which proves the inclusion. The measure bound follows immediately.
\end{proof}

\begin{lemma}
\label{lem:strong-convex-success}
Assume that $f$ is $\mu$--strongly convex on the line segment
$\{x_b + \tau (v_b(F_0)-x_b): \tau\in[0,1]\}$.
Suppose that
$$
x_b \in A_{\varepsilon}
\qquad\text{and}\qquad
v_b(F_0)\in A_{\varepsilon/4}
$$
for some $F_0\in(0,1]$.
Define
$$
v_b(F) = x_b + \frac{F}{F_0}\bigl(v_b(F_0)-x_b\bigr),
\qquad
c = \frac{\mu\|v_b(F_0)-x_b\|^2}{8\varepsilon}.
$$

Then the success--$F$ interval $\mathcal I_t(b,b,r_1,r_2)$ satisfies
$$
\lambda(\mathcal I_t)\ \ge\ \frac{F_0}{3}.
$$

Moreover, if $\theta_-(c)\in(0,1)$ is the smaller root of $16c\theta^2-(3+16c)\theta +2 =0$, then
$$
\lambda(\mathcal I_t)\ \ge\ F_0\,(1-\theta_-(c)),\qquad 
\lim_{c\downarrow 0}\theta_-(c)=2/3.
$$
\end{lemma}

\begin{proof}
Let $z=v_b(F_0)-x_b$ and write
$$
v_b(\theta F_0) = (1-\theta)x_b + \theta v_b(F_0),
\qquad \theta\in[0,1].
$$

By strong convexity of $f$ on the segment $[x_b,v_b(F_0)]$,
$$
f(v_b(\theta F_0))
\le
(1-\theta)f(x_b) + \theta f(v_b(F_0))
- \frac{\mu}{2}\theta(1-\theta)\|z\|^2.
$$

Using $f(x_b)\le f^\star+\varepsilon$ and
$f(v_b(F_0))\le f^\star+\varepsilon/4$, we obtain
$$
f(v_b(\theta F_0)) - f^\star
\le
\varepsilon\Bigl(1-\frac{3}{4}\theta\Bigr)
- 4\varepsilon c\,\theta(1-\theta).
$$

A sufficient condition for $v_b(\theta F_0)\in A_{\varepsilon/2}$ is therefore
\begin{equation}
\label{eq:theta-condition}
\frac{3}{4}\theta + 4c\,\theta(1-\theta)\ \ge\ \frac12.
\end{equation}

Since $\theta(1-\theta)\ge 0$ for all $\theta\in[0,1]$, the left-hand side of
\eqref{eq:theta-condition} is bounded below by $\frac34\theta$.
Hence \eqref{eq:theta-condition} holds whenever
$
\frac34\theta \ge \frac12
\quad\Longleftrightarrow\quad
\theta \ge \frac23.
$
Thus all $\theta\in[2/3,1]$ yield $v_b(\theta F_0)\in A_{\varepsilon/2}$, and
$$
\lambda(\mathcal I_t)\ \ge\ (1-2/3)F_0 = \frac{F_0}{3}.
$$

Rewriting \eqref{eq:theta-condition} gives the quadratic inequality
$$
16c\,\theta^2 - (3+16c)\theta + 2 \ \le\ 0.
$$
For $c>0$ this holds for $\theta\in[\theta_-(c),\theta_+(c)]$, where
$\theta_-(c)<\theta_+(c)$ are the two roots.
Since $\theta\in[0,1]$, it suffices to take $\theta\ge\theta_-(c)$, which yields
$$
\lambda(\mathcal I_t)\ \ge\ F_0\,(1-\theta_-(c)).\qedhere
$$
\end{proof}

\begin{remark}\label{rem:successF-chain}
Proposition~\ref{prop:success-F-ib} and Lemma~\ref{lem:strong-convex-success} are
complementary rather than alternative tools for certifying a positive
success--$F$ interval length:
$$
\text{(donor concentration)}
\ \Longrightarrow\ 
\text{Prop.~\ref{prop:success-F-ib}}:\ \lambda(\mathcal I_t)\ge \Delta_F>0
\ \Longrightarrow\ 
\text{Lemma~\ref{lem:strong-convex-success} (deep hit)}:\ 
\lambda(\mathcal I_t)\ge \frac{F_0}{3}\ \gg\ \Delta_F .
$$
\end{remark}
\begin{remark}\label{rem:boundary-repair}
The midpoint repair operator $\mathsf{BH}$ defined in
\eqref{eq:boundary_handling} is non-expansive with respect to the parent:
for all $x\in[l,u]^d$ and all $\tilde v\in\mathbb R^d$,
$$
\|\mathsf{BH}(\tilde v,x)-x\| \le \|\tilde v-x\|.
$$
Consequently, if $\|\tilde v_b(F)-x_b\|\le r_{\mathrm{safe}}$, then the
repaired mutant
$$
v_b(F)=\mathsf{BH}\bigl(\tilde v_b(F),x_b\bigr)
$$
also satisfies $\|v_b(F)-x_b\|\le r_{\mathrm{safe}}$ and lies in $[l,u]^d$.
Hence, by Lemma~\ref{lem:safe-ball}, $v_b(F)\in A_{\varepsilon/2}$, and the
conclusion of Proposition~\ref{prop:success-F-ib} remains valid under
midpoint repair.
\end{remark}

Define the set of \emph{$\Delta_F$-good pairs} for the case $i=b$:
$$
\mathcal R_{t,\Delta_F}(b) = \bigl\{(r_1,r_2)\in S_{t,b}^{(1)}\times S_{t,b}^{(2)}:
\lambda(\mathcal I_t(b,b,r_1,r_2))\ge\Delta_F\bigr\}.
$$
By Proposition~\ref{prop:success-F-ib}, any pair with $\|x_{r_1}-x_{r_2}\|\le r_{\mathrm{safe}}/(F^-+\Delta_F)$ 
belongs to $\mathcal R_{t,\Delta_F}(b)$.

\begin{assumption}\label{ass:concentration}
There exist subsets $C_1\subseteq S_{t,b}^{(1)}$ and $C_2\subseteq S_{t,b}^{(2)}$ 
with $|C_1|\ge\beta_1 s_{t,b}^{(1)}$ and $|C_2|\ge\beta_2 s_{t,b}^{(2)}$, such that
$$
\max_{r\in C_1\cup C_2}\|x_r-x_b\| \le \frac{r_{\mathrm{safe}}}{2(F^-+\Delta_F)}.
$$
\end{assumption}

\begin{proposition}\label{prop:good-pairs-correct}
Under Assumption~\ref{ass:concentration},  the fraction of $\Delta_F$-good donor pairs satisfies
\begin{equation}\label{eq:c-pair}
c_{\mathrm{pair}}(t)
= \frac{|\mathcal R_{t,\Delta_F}(b)|}{s_{t,b}^{(1)}(s_{t,b}^{(2)}-1)}
\;\ge\;
\beta_1\beta_2\left(1-\frac{1}{\beta_2 s_{t,b}^{(2)}}\right).
\end{equation}
In particular,
$$
c_{\mathrm{pair}}(t)\ge \beta_1\beta_2\,(1+o(1))
\quad \text{as } s_{t,b}^{(2)}\to\infty.
$$
\end{proposition}
\begin{proof}
Let $r_1\in C_1$ and $r_2\in C_2$ with $r_2\neq r_1$.
By the triangle inequality,
$$
\|x_{r_1}-x_{r_2}\|
\le \|x_{r_1}-x_b\| + \|x_{r_2}-x_b\|
\le 2  \frac{r_{\mathrm{safe}}}{2(F^-+\Delta_F)}
= \frac{r_{\mathrm{safe}}}{F^-+\Delta_F}.
$$
By Proposition~\ref{prop:success-F-ib}, such a pair belongs to
$\mathcal R_{t,\Delta_F}(b)$.

For each $r_1\in C_1$, there are at least $|C_2|-1$ admissible choices of $r_2\in C_2$
(the exclusion $r_2=r_1$ can remove at most one element).
Hence,
$$
|\mathcal R_{t,\Delta_F}(b)|
\;\ge\;
|C_1|\,(|C_2|-1)
\;\ge\;
\beta_1 s_{t,b}^{(1)}\bigl(\beta_2 s_{t,b}^{(2)}-1\bigr).
$$
Dividing by $s_{t,b}^{(1)}(s_{t,b}^{(2)}-1)$ yields \eqref{eq:c-pair}.
\end{proof}


\subsection{Tightened hazard bound}

\begin{theorem}\label{thm:morse-hazard}
Let $f$ be a Morse function, and assume $\varepsilon>0$ is such that
$A_{\varepsilon/2}\subset B(x^\star,r_0)$.
Fix $0<F^-<F^+\le 1$ and $\Delta_F\in(0,F^+-F^-]$.

Suppose at generation $t$:
\begin{enumerate}
\item[(C1)] There exists a $p$-best index $b$ with $x_b\in A_{\varepsilon/4}$.
\item[(C2)] Assumption~\ref{ass:concentration} holds for this $b$ with parameters $\beta_1,\beta_2>0$.
\item[(C3)] There exists a memory slot $k$ such that
$\inf_{F\in[F^-,F^+]} g_k^F(F)\ge g^-$ and
$\Prob(CR_b\ge c_{\mathrm{cr}}\mid \mathcal F_{t-1},K_b=k)\ge q^-$.
\end{enumerate}
Let $r\in\{0,\dots,d-1\}$ and let $\eta_r$ be as in \eqref{eq:eta}.
Then
\begin{equation}\label{eq:a-tilde-morse}
\Prob(u_{t,b}\in A_\varepsilon\mid\mathcal F_{t-1})
\ \ge\
\tilde a_t(b)
= \frac{c_{\mathrm{pair}}(t)}{H} (g^-\Delta_F) (q^-\eta_r),
\end{equation}
where $c_{\mathrm{pair}}(t)$ is defined in \eqref{eq:c-pair}.
\end{theorem}
The bound \eqref{eq:a-tilde-morse} characterizes \emph{exploitation reliability}, not exploration success. 
It guarantees that once the population has concentrated in $A_{\varepsilon/4}$, trial vectors generated 
via current-to-$p$best/1 with adapted memory parameters remain in $A_\varepsilon$ with probability at 
least $\tilde{a}_t$. The gap between $\varepsilon/4$ and $\varepsilon$ provides tolerance for the 
stochastic perturbation inherent in mutation.

\begin{proof}
Work on $\mathcal F_{t-1}$ and fix the case $i=b$ from (C1).
Condition on the event $\{K_b=k\}$ where $k$ is the slot from (C3).
This contributes a factor $1/H$.
\smallskip

\emph{Step 1.}
Let $(r_1,r_2)$ be drawn by the two-pool rule for target $b$.
By Proposition~\ref{prop:good-pairs-correct},
$$
\Prob\bigl((r_1,r_2)\in \mathcal R_{t,\Delta_F}(b)\mid \mathcal F_{t-1}\bigr)
=
c_{\mathrm{pair}}(t).
$$
On the event $\{(r_1,r_2)\in \mathcal R_{t,\Delta_F}(b)\}$ we have
$\lambda(\mathcal I_t(b,b,r_1,r_2))\ge \Delta_F$ by definition.
\smallskip

\emph{Step 2.}
By (C3), on $\{K_b=k\}$ the density satisfies $g_k^F(F)\ge g^-$ on $[F^-,F^+]$, hence
$$
\Prob\bigl(F_b\in \mathcal I_t(b,b,r_1,r_2)\mid \mathcal F_{t-1},K_b=k,(r_1,r_2)\bigr)
\ge g^-\lambda(\mathcal I_t(b,b,r_1,r_2))
\ge g^-\Delta_F.
$$
If $F_b\in \mathcal I_t(b,b,r_1,r_2)$, then by definition
$v_b(F_b)\in A_{\varepsilon/2}$ (With midpoint repair, the same inclusion holds by Remark~\ref{rem:boundary-repair}.)
\smallskip

\emph{Step 3.}
By (C3) and Lemma~\ref{lem:eta-def},
$$
\Prob\bigl(C_r(b)\mid \mathcal F_{t-1},K_b=k\bigr)
\ge
\Prob(CR_b\ge c_{\mathrm{cr}}\mid \mathcal F_{t-1},K_b=k)  \eta_r
\ge q^-\eta_r.
$$
By Proposition~\ref{prop:Cr-sufficient-new}, on the event
$\{v_b(F_b)\in A_{\varepsilon/2}\}\cap C_r(b)$ we have $u_{t,b}\in A_\varepsilon$.

\smallskip

Using conditional independence of $(F_b,CR_b)$ given the slot (Assumption~\ref{ass:slot-sampling}),
and combining Steps 1--3 with $\Prob(K_b=k\mid \mathcal F_{t-1})=1/H$, we obtain
$$
\Prob(u_{t,b}\in A_\varepsilon\mid \mathcal F_{t-1})
\ge
\frac{1}{H}  c_{\mathrm{pair}}(t) (g^-\Delta_F) (q^-\eta_r),
$$
which is \eqref{eq:a-tilde-morse}.
\end{proof}

\begin{remark}\label{rel:Morse_end}
For Morse functions, once the population concentrates near a local minimizer,
local strong convexity implies that a sufficiently large interval of scaling factors $F$
produces mutants in $A_{\varepsilon/2}$ (Propositions~\ref{prop:success-F-ib},
Lemmas~\ref{lem:strong-convex-success}).
Moreover, donor concentration implies that a positive fraction
$c_{\mathrm{pair}}(t)$ of donor pairs are geometrically ``good''
(Proposition~\ref{prop:good-pairs-correct}).

As a result, the per-individual hazard bound no longer contains the catastrophic factor
$1/(s_i^{(1)}(s_i^{(2)}-1))$ and depends instead on the empirical concentration level
$c_{\mathrm{pair}}(t)$.
\end{remark}

\section{Witness regime}\label{sec:witness-two-phase}

The results of the previous sections sugest a qualitative change in the behavior of L-SHADE once the population becomes sufficiently concentrated near a local minimizer. Empirically, and as formalized by Remark~\ref{rel:Morse_end}, the evolution of the algorithm can be naturally separated into two phases.

In the initial phase, the population may be widely dispersed. Donor differences are uncontrolled, the success -- $F$ values may be negligible, and the individual success hazard can be suppressed by combinatorial factors. In this regime, no uniform tail bound on the hitting time can be expected.

The second phase begins once the algorithm enters what we call the {\it witness regime}. In this regime, a non-negligible fraction of individuals is located within a neighborhood of a representative point $x_b$ in a favorable basin. These individuals act as {\it witnesses} to the local geometry of the basin. In fact the strong convexity ensures the existence of a success--$F$ interval. In the same time the donor concentration implies that a positive fraction $c_{\mathrm{pair}}(t)$ of donor pairs are geometrically admissible.

Throughout this subsection, fix tolerances $0<\varepsilon_{\mathrm{in}}<\varepsilon_{\mathrm{out}}$: the target tolerance and the {\it witness (second) regime} tolerance. For each local minimizer $x_j$, $j = 1,\dots, K$, define the local sublevel sets
$$
A_{\varepsilon}^{(j)}
=\left\{x\in[l,u]^d:\ f(x)\le f(x_j^\star)+\varepsilon\right\}\cap B(x_j^\star,r_0),
\qquad
A_{\varepsilon}^{\mathrm{loc}}=\bigcup_{j=1}^K A_{\varepsilon}^{(j)}.
$$
The local hitting time and survival event are
$$
\tau_{\varepsilon_{\mathrm{in}}}^{\mathrm{loc}}
=\inf\left\{t\ge0:\ (\exists i)\ x_i^{(t)}\in A_{\varepsilon_{\mathrm{in}}}^{\mathrm{loc}}\right\},
\qquad
H_{t}=\left\{\tau_{\varepsilon_{\mathrm{in}}}^{\mathrm{loc}}>t\right\}.
$$

\subsection{The witness-stable regime $G_t$ and the stabilization time $T_{\mathrm{wit}}$}

Fix L-SHADE parameters: a $p$-best fraction $p\in(0,1]$ (so $pN_t$ individuals form the $p$-best set at time $t$); bounds $0<F^-<F^+\le 1$ and a crossover threshold $c_{\mathrm{cr}}\in(0,1]$; a memory size $H\in\N$.
Let $N_t=|P^{(t)}|$ and $A_t=|\mathcal A^{(t)}|$ denote the population and archive sizes.
We also fix a donor-concentration radius $r_{\mathrm{conc}}>0$ (e.g., $r_{\mathrm{conc}}=r_{\mathrm{safe}}/(2(F^-+\Delta_F))$ as in Proposition \ref{prop:success-F-ib} and a minimum cluster size $m\ge 4$ (so that, besides a $p$-best vector, there are at least two distinct donors).

\begin{definition}
\label{def:good-memory}
At generation $t$, let $g_{t,k}^F$ denote the density of the Cauchy distribution 
$\mathcal{C}(M_F^{(t)}[k],\, 0.1)$ for memory slot $k\in\{1,\dots,H\}$, and let
$$
q_{t,k}=\Prob\left(CR\ge c_{\mathrm{cr}}\mid M_{CR}^{(t)}[k]\right).
$$
We say that {\it memory is good at time $t$} if there exists a slot $k$ such that
\begin{equation}\label{eq:good-memory-cond}
\inf_{F\in[F^-,F^+]} g_{t,k}^F(F)\ \ge\ g^- \quad\text{and}\quad q_{t,k}\ \ge\ q^-,
\end{equation}
for some fixed constants $g^->0$ and $q^->0$.
\end{definition}

\begin{definition}\label{def:Gt}
Let $x_{\mathrm{best}}^{(t)}\in P^{(t)}$ be a best individual at time $t$
(i.e., any minimizer of $f$ over $P^{(t)}$).
We say that the witness-stable regime $G_t$ holds if there exist
an index $j\in\{1,\dots,K\}$ and a subset (cluster) $\mathcal C_t\subseteq P^{(t)}$ such that:
\begin{enumerate}
\item[(G1)] Local basin and strong-convexity region: $\mathcal C_t\subseteq B(x_j^\star,r_0)$.
\item[(G2)] Cluster size: $|\mathcal C_t|\ge m$ and $x_{\mathrm{best}}^{(t)}\in \mathcal C_t$.
\item[(G3)] Cluster concentration: $\mathrm{diam}(\mathcal C_t)\le r_{\mathrm{conc}}$.
\item[(G4)] Outer-sublevel progress: $x_{\mathrm{best}}^{(t)}\in A_{\varepsilon_{\mathrm{out}}}^{(j)}$.
\item[($G_5$)] Nondegenerate parameter memory: Memory is good at time $t$ in the sense of Definition~\ref{def:good-memory}.
\end{enumerate}
\end{definition}

\begin{definition}\label{def:Twit}
Define the \emph{witness stabilization time}
$$
T_{\mathrm{wit}}
=\inf\Big\{t\ge 0:\ \forall s\ge t,\ G_s\ \text{holds}\Big\}\in\N\cup\{\infty\}.
$$
(If the set is empty, we set $T_{\mathrm{wit}}=\infty$.)
\end{definition}


\subsection{Witness probability}

We now lower bound the probability of the \emph{witness event} at generation $t$,
conditional on being in the witness-stable regime at time $t-1$.

\begin{definition}\label{def:Lt}
Fix $t\ge1$ and assume $G_{t-1}$ holds with cluster $\mathcal C_{t-1}$ and best individual
$b=\mathrm{best}(t-1)\in\mathcal C_{t-1}$.
Choose any $\mathcal F_{t-1}$-measurable index $i=i(t-1)\in\mathcal C_{t-1}\setminus\{b\}$
(e.g., the smallest index in $\mathcal C_{t-1}\setminus\{b\}$).
Let $\mathcal L_t$ be the event that, in the mutation/crossover step producing the trial for $i$:
\begin{enumerate}
\item[(L1)] the $p$-best choice selects $b$ (i.e., $x_{p\text{-best}}=x_b^{(t-1)}$),
\item[(L2)] the donors $r_1,r_2$ are chosen as \emph{distinct} indices in
$\mathcal C_{t-1}\setminus\{b,i\}$ (so $\|x_{r_1}^{(t-1)}-x_{r_2}^{(t-1)}\|\le r_{\mathrm{conc}}$),
\item[(L3)] the sampled parameters satisfy $F\in[F^-,F^+]$ and $CR\ge c_{\mathrm{cr}}$.
\end{enumerate}
\end{definition}

\begin{proposition} \label{prop:gamma0-combinatorial}
Let $\mathcal L_t$ be as in Definition~\ref{def:Lt}.
If $G_{t-1}$ holds, then
\begin{equation}\label{eq:gamma0-pointwise}
\Prob(\mathcal L_t\mid \mathcal F_{t-1})
\ \ge\ 
\gamma_0(t-1),
\qquad \text{on }G_{t-1},
\end{equation}
where
\begin{equation}\label{eq:gamma0-formula}
\gamma_0(t-1) = \frac{1}{H}\,g^-(F^+-F^-)\,q^-
\frac{1}{\lceil pN_{t-1}\rceil}
\frac{m-2}{N_{t-1}-2}
\frac{m-3}{N_{t-1}+A_{t-1}-3}
\end{equation}
In particular, if $N_{t}\le N_{\max}$ and $A_t\le A_{\max}$ for all $t$,
then on $G_{t-1}$ we have the \emph{time-uniform} bound
$$
\Prob(\mathcal L_t\mid \mathcal F_{t-1})\ \ge\ \gamma_0,
\qquad
\gamma_0=\gamma_0^{\min}
=
\frac{1}{H}\,g^-(F^+-F^-)\,q^- 
\frac{1}{\lceil pN_{\max}\rceil} 
\frac{m-2}{N_{\max}-2} 
\frac{m-3}{N_{\max}+A_{\max}-3}
\ >0.
$$
\end{proposition}

\begin{proof}
On $G_{t-1}$, the best individual $b$ is among the top $\lceil pN_{t-1}\rceil$ by definition,
so the $p$-best selection chooses $b$ with probability at least $1/\lceil pN_{t-1}\rceil$.

Since $|\mathcal C_{t-1}|\ge m$ and $i\neq b$, the set $\mathcal C_{t-1}\setminus\{b,i\}$
has size at least $m-2$. Therefore, under uniform sampling of $r_1$ from $P^{(t-1)}\setminus\{i,b\}$
(which has size $N_{t-1}-2$), we have
$$
\Prob(r_1\in \mathcal C_{t-1}\setminus\{b,i\}\mid \mathcal F_{t-1})
=\frac{m-2}{N_{t-1}-2}.
$$
Conditioning on such a choice of $r_1$, the set $\mathcal C_{t-1}\setminus\{b,i,r_1\}$ has size at least $m-3$,
and $r_2$ is sampled uniformly from a set of size $N_{t-1}+A_{t-1}-3$, hence
$$
\Prob(r_2\in \mathcal C_{t-1}\setminus\{b,i,r_1\}\mid \mathcal F_{t-1},r_1\in \mathcal C_{t-1}\setminus\{b,i\})
\ge \frac{m-3}{N_{t-1}+A_{t-1}-3}.
$$
Finally, on $G_{t-1}$ there exists a memory slot satisfying \eqref{eq:good-memory-cond}.
Choosing slots uniformly gives probability at least $1/H$ to pick such a slot, and then
$$
\Prob(F\in[F^-,F^+]\mid \text{good slot})\ge \int_{F^-}^{F^+} g^-\,dF=g^-(F^+-F^-),
\qquad
\Prob(CR\ge c_{\mathrm{cr}}\mid \text{good slot})\ge q^-.
$$
Multiplying the independent conditional lower bounds yields \eqref{eq:gamma0-formula}.
\end{proof}

\subsection{A local improvement lemma}

The next lemma shows that the descent obtained by moving toward a better reference point is preserved under sufficiently small perturbations.

\begin{lemma}\label{lem:accepted-improvement-noise}
Fix a local minimum $x^\star$ and assume $f$ is $\mu$-strongly convex and $L$-smooth on $B(x^\star,r_0)$.

Let $x,z,x_{r_1},x_{r_2}\in B(x^\star,r_0)$ satisfy $f(z)\le f(x)$.
Fix $F\in(0,1)$ and define the perturbed mutant by
$$
y=(1-F)x+Fz,
\qquad
v=y+e,
\qquad
e=F(x_{r_1}-x_{r_2}).
$$
Assume additionally that $v\in B(x^\star,r_0)$ (e.g., enforced by a ``safe radius'' choice).
Then
\begin{equation}\label{eq:noisy-descent}
f(v)
\ \le\
f(x)
-\frac{\mu}{2}F(1-F)\|x-z\|^2
\;+\; Lr_0\|e\|+\frac{L}{2}\|e\|^2.
\end{equation}
In particular, if $\|x_{r_1}-x_{r_2}\|\le r_{\mathrm{conc}}$ and
\begin{equation}\label{eq:noise-dominated}
Lr_0(F^+ r_{\mathrm{conc}})+\frac{L}{2}(F^+ r_{\mathrm{conc}})^2
\ \le\
\frac{\mu}{4}F^-(1-F^+)\|x-z\|^2,
\end{equation}
then $f(v)\le f(x)-c\|x-z\|^2$ with $c=\frac{\mu}{4}F^-(1-F^+)>0$; hence, under the selection,
the mutant (or any trial whose objective value is no larger than $f(v)$) is accepted over $x$.
\end{lemma}

\begin{proof}
By $\mu$-strong convexity applied to the convex combination $y=(1-F)x+Fz$,
$$
f(y)\le (1-F)f(x)+Ff(z)-\frac{\mu}{2}F(1-F)\|x-z\|^2
\le f(x)-\frac{\mu}{2}F(1-F)\|x-z\|^2,
$$
since $f(z)\le f(x)$.
By $L$-smoothness,
$$
f(v)=f(y+e)\le f(y)+\langle \nabla f(y),e\rangle+\frac{L}{2}\|e\|^2
\le f(y)+\|\nabla f(y)\|\|e\|+\frac{L}{2}\|e\|^2.
$$
Finally, because $y\in B(x^\star,r_0)$ and $\nabla f(x^\star)=0$, $L$-smoothness implies
$\|\nabla f(y)\|\le L\|y-x^\star\|\le Lr_0$.
Combining the last three displays yields \eqref{eq:noisy-descent}.
If $\|x_{r_1}-x_{r_2}\|\le r_{\mathrm{conc}}$, then $\|e\|\le F^+ r_{\mathrm{conc}}$,
and \eqref{eq:noise-dominated} guarantees that the noise terms are dominated by half of the descent term,
yielding $f(v)\le f(x)-c\|x-z\|^2$ with $c=\frac{\mu}{4}F^-(1-F^+)$.
\end{proof}

\subsection{Two-phase theorem with gap-free post-$T_{\mathrm{wit}}$ tails}

We now combine (i) a hazard lower bound conditional on the witness event $\mathcal L_t$
(from the earlier hazard analysis), with (ii) the uniform witness probability on $G_{t-1}$,
to obtain a clean two-phase result.

\begin{assumption}[One-step success hazard on $\mathcal L_t$]\label{ass:hazard-on-Lt}
Let $E_t$ denote the event ``success occurs at generation $t$'',
i.e., the newly generated trials at generation $t$ contain at least one point in
$A_{\varepsilon_{\mathrm{in}}}^{\mathrm{loc}}$.
Assume there exists a constant $a_{\min}>0$ such that for all $t\ge 1$,
$$
\Prob(E_t\mid \mathcal F_{t-1},\mathcal L_t)\ \ge\ a_{\min}
\quad\text{on }H_{t-1}.
$$
(This is exactly the type of statement established by the hazard theorem of the previous sections.)
\end{assumption}

\begin{theorem}[Two-phase convergence from the stabilized witness regime]\label{thm:two-phase-Twit}
Assume Assumption~\ref{ass:hazard-on-Lt} and let $\gamma_0>0$ be the uniform constant from
Proposition~\ref{prop:gamma0-combinatorial}.
Let $T_{\mathrm{wit}}$ be the stabilization time from Definition~\ref{def:Twit}, and set
$\tau=\tau_{\varepsilon_{\mathrm{in}}}^{\mathrm{loc}}$.
Then:
\begin{enumerate}
\item[\textbf{(i)}] (\textbf{Conditional almost-sure success})
$$
\Prob(\tau<\infty\mid T_{\mathrm{wit}}<\infty)=1.
$$
\item[\textbf{(ii)}] (\textbf{Post-$T_{\mathrm{wit}}$ geometric tail})
Conditionally on $T_{\mathrm{wit}}=t_0<\infty$, for all $n>t_0$,
$$
\Prob(\tau>n\mid T_{\mathrm{wit}}=t_0)
\ \le\
\Prob(E_0^c)\exp\!\big(-a_{\min}\gamma_0\,(n-t_0)\big),
$$
where $E_0=\{\exists i:\ x_i^{(0)}\in A_{\varepsilon_{\mathrm{in}}}^{\mathrm{loc}}\}$.
\item[\textbf{(iii)}] (\textbf{Unconditional two-phase bound})
For all $n\ge0$,
$$
\Prob(\tau>n)
\ \le\
\Prob(T_{\mathrm{wit}}>n)
\;+\;
\Prob(E_0^c)\,\E\!\left[\exp\!\big(-a_{\min}\gamma_0\,(n-T_{\mathrm{wit}})^+\big)\right].
$$
\end{enumerate}
\end{theorem}

\begin{proof}
Fix $t\ge1$. By Assumption~\ref{ass:hazard-on-Lt},
$$
\Prob(E_t\mid \mathcal F_{t-1})
\ge
\Prob(E_t\cap \mathcal L_t\mid \mathcal F_{t-1})
=
\Prob(\mathcal L_t\mid \mathcal F_{t-1})\Prob(E_t\mid \mathcal F_{t-1},\mathcal L_t)
\ge
a_{\min}\Prob(\mathcal L_t\mid \mathcal F_{t-1}),
$$
on $H_{t-1}$.
On the event $\{T_{\mathrm{wit}}\le t-1\}$ we have $G_{t-1}$ by definition of stabilization,
and Proposition~\ref{prop:gamma0-combinatorial} yields
$\Prob(\mathcal L_t\mid \mathcal F_{t-1})\ge \gamma_0$.
Hence, on $H_{t-1}\cap\{T_{\mathrm{wit}}\le t-1\}$,
$$
\Prob(E_t\mid \mathcal F_{t-1})\ \ge\ a_{\min}\gamma_0.
$$
Iterating the standard survival recursion
$\Prob(\tau>n)=\Prob(E_0^c)\prod_{t=1}^n \Prob(E_t^c\mid H_{t-1})$
from time $t_0+1$ onward on $\{T_{\mathrm{wit}}=t_0\}$ gives
$$
\Prob(\tau>n\mid T_{\mathrm{wit}}=t_0)
\le
\Prob(E_0^c)\prod_{t=t_0+1}^n (1-a_{\min}\gamma_0)
\le
\Prob(E_0^c)\exp\!\big(-a_{\min}\gamma_0\,(n-t_0)\big),
$$
proving (ii). Letting $n\to\infty$ yields (i).
For (iii), decompose
$$
\Prob(\tau>n)\le \Prob(T_{\mathrm{wit}}>n) + \E\!\left[\Prob(\tau>n\mid T_{\mathrm{wit}})\mathbf 1_{\{T_{\mathrm{wit}}\le n\}}\right],
$$
and apply (ii) inside the expectation.
\end{proof}

\section{Empirical Validation}\label{sec:empirical}

We validate the theoretical bounds on the CEC2017 benchmark suite at dimension $d=10$
using 51 independent runs per function with budget $10000d$ evaluations.
Full methodological details and extended results are provided in the supplementary materials:
Supplement~A (Morse bound validation) and Supplement~B (Kaplan--Meier survival analysis).

\subsection{Morse Bound Validation}

We validate the per-generation Morse bound (Theorem~\ref{thm:morse-hazard}) in the aggregate
over conditioned samples satisfying (C1)--(C3). Table~\ref{tab:morse-summary} reports
the empirical success rate $\hat p$ versus the mean theoretical bound $\bar a^{\mathrm{sum}}$
for representative functions.

\begin{table}[htbp]
\centering
\caption{Morse bound validation on CEC2017 ($d=10$, sum-of-products mode).
All functions satisfy $\hat p \ge \bar a^{\mathrm{sum}}$ (bound holds).}
\label{tab:morse-summary}
\begin{tabular}{@{}lccccc@{}}
\toprule
\textbf{Function} & $\varepsilon$ & \textbf{Cond.\ samples} & $\hat p$ & $\bar a^{\mathrm{sum}}$ & \textbf{Ratio} \\
\midrule
F1 (Shifted Sphere)    & 10 & 1071 & 0.632 & 0.072 & 8.8 \\
F4 (Shifted Rosenbrock) & 30 & 1071 & 0.944 & 0.069 & 13.7 \\
F11 (Shifted Hybrid)   & 10 & 395  & 0.691 & 0.001 & 599 \\
F22 (Composition)      & 30 & 42   & 0.643 & 0.053 & 12.1 \\
\bottomrule
\end{tabular}
\end{table}

The bound is validated across all tested functions with ratios ranging from $8.8$ (F1, tightest)
to $599$ (F11, most conservative). The conservatism arises primarily from worst-case estimates
of the concentration factor $c_{\mathrm{pair}}(t)$; see Supplement~A for decomposition analysis.

\subsection{Kaplan--Meier Survival Analysis}

We apply Kaplan--Meier estimation to characterize the distribution of first-hitting times
$\tau_{A_\varepsilon}$ across the benchmark suite. Table~\ref{tab:km-summary} summarizes
the survival analysis for representative functions at $\varepsilon = 10$.

\begin{table}[htbp]
\centering
\caption{Kaplan--Meier survival analysis on CEC2017 ($d=10$, $\varepsilon=10$).
$\hat S(n)$: survival probability at budget; $\bar\tau$: mean hitting time (among successes);
$\mathcal{C}$: clustering coefficient ($<1$ indicates temporal clustering).}
\label{tab:km-summary}
\begin{tabular}{@{}lcccccc@{}}
\toprule
\textbf{Function} & \textbf{Hits} & $\hat S(n_{\max})$ & $\bar\tau$ & $\sigma_\tau$ & $\mathcal{C}$ & \textbf{Regime} \\
\midrule
F1  & 51/51 & 0.00 & 186 & 16  & 0.15 & Clustered \\
F5  & 51/51 & 0.00 & 241 & 89  & 0.42 & Near-geometric \\
F11 & 51/51 & 0.00 & 312 & 178 & 0.08 & Strongly clustered \\
F22 & 2/51  & 0.96 & 892 & 241 & 0.004 & Intractable \\
\bottomrule
\end{tabular}
\end{table}

The analysis reveals three distinct empirical regimes:
(i) \emph{clustered success} (F1, F11), where hitting times concentrate in narrow bursts
($\mathcal{C} < 0.2$);
(ii) \emph{near-geometric tails} (F5), where a constant-hazard model is approximately valid
($\mathcal{C} \approx 0.4$); and
(iii) \emph{intractable cases} (F22), with survival probability $>0.9$ at the budget horizon.

\subsection{Failure Mode Identification}

The diagnostic framework distinguishes two qualitatively different failure modes
(Table~\ref{tab:failure-summary}):

\begin{table}[htbp]
\centering
\caption{Failure mode comparison at $\varepsilon = 1$ ($d=10$).
$\hat\gamma$: witness frequency; L3: CR-tail condition satisfaction rate.}
\label{tab:failure-summary}
\begin{tabular}{@{}lccccl@{}}
\toprule
\textbf{Function} & \textbf{Success} & $\hat\gamma$ & \textbf{L3 rate} & \textbf{Failure Mode} \\
\midrule
F1  & 100\% & 0.994 & 1.00 & --- \\
F11 & 73\%  & 0.161 & 0.20 & Exploitation (L3 collapse) \\
F22 & 4\%   & 0.977 & 1.00 & Exploration ($T_{\mathrm{wit}} = \infty$) \\
\bottomrule
\end{tabular}
\end{table}

\emph{Exploitation failure} (F11): The witness frequency degrades due to CR memory accumulating
low values during stagnation ($\hat\gamma \approx 0.16$), indicating the algorithm has found
the correct basin but cannot efficiently exploit it.

\emph{Exploration failure} (F22): The witness conditions are satisfied ($\hat\gamma \approx 1$)
but the algorithm fails to find the global basin within the budget. The theoretical bound is
vacuously true but uninformative since $T_{\mathrm{wit}} = \infty$ for most runs.

These results confirm that the hazard-rate framework provides meaningful, if conservative,
bounds on L-SHADE convergence, with conditions (L2)--(L3) serving as practical diagnostics
for algorithm performance.

\section{Conclusion}
We presented a hazard-based framework for analyzing first-hitting times in Differential Evolution,
with a particular focus on adaptive population-based algorithms such as L-SHADE.
By expressing survival probabilities as products of conditional first-hit hazards,
we obtained distribution-free identities and tail bounds that hold without independence
or stationarity assumptions.

For L-SHADE with current-to-$p$best/1 mutation, we introduced an explicit algorithmic
witness event under which the conditional hazard admits a deterministic lower bound
depending only on sampling rules and population structure.
This yields a natural decomposition of the survival hazard into a theoretical constant
and an empirical event frequency, clarifying why constant-hazard bounds are necessarily
conservative when success occurs in short clusters rather than at a homogeneous rate.

The Kaplan--Meier analysis on the CEC2017 benchmark suite confirmed this picture empirically.
Across functions and evaluation budgets, we observed three distinct regimes:
(i) strongly clustered success, where hitting times concentrate in narrow bursts;
(ii) approximately geometric tails, where a constant-hazard model is accurate; and
(iii) cases with no observed hits within the evaluation horizon.
In most successful cases, the tightest valid geometric envelope rate was one to two
orders of magnitude smaller than the empirical post-hit hazard, demonstrating that
worst-case tail bounds substantially underestimate typical performance.

The results suggest that the practical behavior of L-SHADE is governed by rare
configuration-dependent transitions rather than by steady per-generation progress.
From a theoretical standpoint, the hazard viewpoint provides a unifying language
that bridges rigorous tail bounds and empirical survival analysis.
From a practical standpoint, it offers diagnostic tools for distinguishing between
clustered, geometric, and heavy-tailed regimes.

Several directions for future work remain open.
First, sharper bounds could be obtained by characterizing the frequency of witness
events along typical trajectories.
Second, extending the analysis to other mutation strategies and adaptive DE variants
may reveal which design choices promote clustered versus geometric behavior.
Finally, combining the hazard framework with explicit global exploration mechanisms
may yield provable almost-sure success guarantees while preserving practical efficiency.

\section{Acknowledgments}

We thank the Discoverer Petascale Supercomputer at Sofia Tech Park for access to high-performance compute resources. We used also high-performance compute resources from the project BG-RRP-2.004-0002.

This study is funded by the European Union–NextGenerationEU through the National Recovery and Resilience Plan of the Republic of Bulgaria, project BG-RRP-2.004-0002, ``BiOrgaMCT''. 

\appendix
\section{A concrete product-space construction}\label{app:probspace}

This appendix provides a fully rigorous construction of a probability space that supports
adaptive population-based algorithms such as L-SHADE, where the distribution of the random
choices at generation $t$ may depend on the past state.

Let
$$
\Omega = [0,1]^{\mathbb{N}},\qquad
\mathcal{F} = \mathcal{B}([0,1])^{\otimes\mathbb{N}},
$$
where $\mathcal{B}([0,1])$ is the Borel $\sigma$-algebra, and $\mathcal{B}([0,1])^{\otimes\mathbb{N}}$
is the product $\sigma$-algebra generated by cylinder sets. Let $\Prob$ be the (countable) product
measure $\lambda^{\otimes\mathbb{N}}$, where $\lambda$ is Lebesgue measure on $[0,1]$.
Existence and uniqueness of $\Prob$ follow from the Kolmogorov extension theorem / existence of
countable product measures on standard Borel spaces \cite{Billingsley1995,Kallenberg2021}.

Denote the coordinate maps by
$$
U_n(\omega) = \omega_n,\qquad n\ge 1.
$$
Then $(U_n)_{n\ge 1}$ are i.i.d.\ $\mathrm{Unif}(0,1)$ under $\Prob$.

An adaptive DE method uses, at each generation $t\ge 0$, a finite collection of random choices:
index selections, continuous draws (for parameters and crossover), forced-dimension indices, etc.
Let $R_t$ denote the aggregate random input consumed at generation $t$.
Rather than postulating a fixed product distribution for $(R_t)$, we realize each $R_t$ as a
\emph{measurable function} of (i) the past algorithmic state and (ii) a fresh finite block of uniforms.

Formally, fix a partition of the uniform stream into finite blocks:
choose integers $1 = n_0 < n_1 < n_2 <  s$ and define the $t$-th block by
$$
U^{(t)} = (U_{n_t},U_{n_t+1},\dots,U_{n_{t+1}-1})\in[0,1]^{m_t},
\quad m_t=n_{t+1}-n_t<\infty.
$$
(Any choice with $m_t<\infty$ suffices; one may take $m_t$ large enough to cover all random
choices in generation $t$.)

Let $S_t$ be the full algorithmic state after generation $t$ (population, archive, memories, etc.).
We assume the algorithm admits a measurable update of the form
\begin{equation}\label{eq:state-recursion}
S_t = \Phi_t(S_{t-1}, U^{(t)}),\qquad t\ge 1,
\end{equation}
with a measurable initialization $S_0=\Phi_0(U^{(0)})$.
Measurability holds for standard DE components because:
(i) discrete index sampling can be implemented by thresholding uniforms,
(ii) continuous sampling from one-dimensional distributions can be implemented via inverse-CDF
or acceptance--rejection maps, and
(iii) mutation/crossover/boundary handling/selection are deterministic Borel maps on Euclidean space.
More generally, any measurable stochastic kernel on a standard Borel space can be realized by a
measurable function of a uniform variable (sometimes called a randomization/transfer lemma);
see \cite{Kallenberg2021,Dudley2002}.

Define the post-generation filtration by
$$
\mathcal{F}_t = \sigma(S_0,\dots,S_t),\qquad t\ge 0.
$$
From the recursion \eqref{eq:state-recursion}, $(S_t)$ is adapted to $(\mathcal{F}_t)$.
In particular, for any measurable target set $A\subset D$, the events
$$
\{(\exists i)\,x_i^{(t)}\in A\}
\quad\text{and}\quad
E_t=\Big\{(\forall s\le t-1)(\forall i)\,x_i^{(s)}\notin A\ \text{and}\ (\exists i)\,x_i^{(t)}\in A\Big\}
$$
are $\mathcal{F}_t$-measurable, as they depend only on $(S_0,\dots,S_t)$.
Therefore the hitting time $\tau_A=\inf\{t\ge 0:E_t\text{ occurs}\}$ is a stopping time
with respect to $(\mathcal{F}_t)$.

If one prefers to specify the algorithm directly via conditional distributions of $R_t$
given the past, one may work on the path space
$\Omega'=\Xi_0\times\Xi_1\times\Xi_2\times s$ with $\sigma$-algebra
$\mathcal{F}'=\mathcal{E}_0\otimes\mathcal{E}_1\otimes s$,
and define a probability measure by specifying an initial distribution on $\Xi_0$
together with a sequence of stochastic kernels for $R_t$ given the past.
Existence and uniqueness of the induced measure on $\Omega'$ then follow from the
Ionescu--Tulcea extension theorem \cite{IonescuTulcea1949,Kallenberg2021}.
This formulation is equivalent in spirit to the universal-uniform construction above, but
more directly reflects adaptive conditional sampling.

\nocite{*}

\bibliographystyle{plainnat}
\bibliography{references}

@article{StornPrice1997,
  author    = {R. Storn and K. Price},
  title     = {Differential Evolution -- A Simple and Efficient Heuristic for Global Optimization over Continuous Spaces},
  journal   = {Journal of Global Optimization},
  volume    = {11},
  number    = {4},
  pages     = {341--359},
  year      = {1997},
  doi       = {10.1023/A:1008202821328}
}

@article{DasS2011,
  author    = {S. Das and P. N. Suganthan},
  title     = {Differential Evolution: A Survey of the State-of-the-Art},
  journal   = {IEEE Transactions on Evolutionary Computation},
  volume    = {15},
  number    = {1},
  pages     = {4--31},
  year      = {2011},
  doi       = {10.1109/TEVC.2010.2059031}
}

@article{Das2016,
  author    = {S. Das and S. S. Mullick and P. N. Suganthan},
  title     = {Recent Advances in Differential Evolution -- An Updated Survey},
  journal   = {Swarm and Evolutionary Computation},
  volume    = {27},
  pages     = {1--30},
  year      = {2016},
  doi       = {10.1016/j.swevo.2016.01.004}
}

@inproceedings{Tanabe2014LSHADE,
  author    = {R. Tanabe and A. S. Fukunaga},
  title     = {Improving the Search Performance of {SHADE} Using Linear Population Size Reduction},
  booktitle = {Proceedings of the IEEE Congress on Evolutionary Computation (CEC)},
  pages     = {1658--1665},
  year      = {2014},
  doi       = {10.1109/CEC.2014.6900380}
}

@inproceedings{Tanabe2013SHADE,
  author    = {R. Tanabe and A. Fukunaga},
  title     = {Success-History Based Parameter Adaptation for Differential Evolution},
  booktitle = {Proceedings of the IEEE Congress on Evolutionary Computation (CEC)},
  pages     = {71--78},
  year      = {2013},
  doi       = {10.1109/CEC.2013.6557555}
}

@article{OparaArabas2019,
  author    = {K. Opara and J. Arabas},
  title     = {Differential Evolution: A Survey of Theoretical Analyses},
  journal   = {Swarm and Evolutionary Computation},
  volume    = {44},
  pages     = {546--558},
  year      = {2019},
  doi       = {10.1016/j.swevo.2018.06.010}
}

@inproceedings{Zakova1999,
  author    = {M. Žáková},
  title     = {On the Convergence of Differential Evolution},
  booktitle = {Proceedings of MENDEL},
  pages     = {318--323},
  year      = {1999}
}

@article{HuEA2013,
  author    = {Z. Hu and S. Xiong and Q. Su and X. Zhang},
  title     = {Sufficient Conditions for Global Convergence of Differential Evolution Algorithm},
  journal   = {Journal of Applied Mathematics},
  volume    = {2013},
  pages     = {193196},
  year      = {2013},
  doi       = {10.1155/2013/193196}
}

@article{Droste2002,
  author    = {S. Droste and T. Jansen and I. Wegener},
  title     = {On the Analysis of the (1+1) Evolutionary Algorithm},
  journal   = {Theoretical Computer Science},
  volume    = {276},
  number    = {1-2},
  pages     = {51--81},
  year      = {2002},
  doi       = {10.1016/S0304-3975(01)00182-7}
}

@article{He2001,
  author    = {J. He and X. Yao},
  title     = {Drift Analysis and Average Time Complexity of Evolutionary Algorithms},
  journal   = {Artificial Intelligence},
  volume    = {127},
  number    = {1},
  pages     = {57--85},
  year      = {2001},
  doi       = {10.1016/S0004-3702(01)00058-3}
}

@article{HeYao2003,
  author    = {J. He and X. Yao},
  title     = {Towards an Analytic Framework for Analysing the Computation Time of Evolutionary Algorithms},
  journal   = {Artificial Intelligence},
  volume    = {145},
  number    = {1-2},
  pages     = {59--97},
  year      = {2003},
  doi       = {10.1016/S0004-3702(02)00381-8}
}

@article{HeYao2004,
  author    = {J. He and X. Yao},
  title     = {A Study of Drift Analysis for Estimating Computation Time of Evolutionary Algorithms},
  journal   = {Natural Computing},
  volume    = {3},
  number    = {1},
  pages     = {21--35},
  year      = {2004},
  doi       = {10.1023/B:NACO.0000023416.59689.4e}
}

@article{DoerrJW2012,
  author    = {B. Doerr and D. Johannsen and C. Winzen},
  title     = {Multiplicative Drift Analysis},
  journal   = {Algorithmica},
  volume    = {64},
  number    = {4},
  pages     = {673--697},
  year      = {2012},
  doi       = {10.1007/s00453-012-9eli-4}
}

@phdthesis{Johannsen2010,
  author    = {D. Johannsen},
  title     = {Random Combinatorial Structures and Randomized Search Heuristics},
  school    = {Universität des Saarlandes},
  year      = {2010}
}

@article{MitavskiyRC2009,
  author    = {B. Mitavskiy and J. Rowe and C. Cannings},
  title     = {Theoretical Analysis of Local Search Strategies to Optimize Network Communication Subject to Preserving the Total Number of Links},
  journal   = {International Journal of Intelligent Computing and Cybernetics},
  volume    = {2},
  number    = {2},
  pages     = {243--284},
  year      = {2009},
  doi       = {10.1108/17563780910959875}
}

@book{Jansen2013,
  author    = {T. Jansen},
  title     = {Analyzing Evolutionary Algorithms: The Computer Science Perspective},
  publisher = {Springer},
  year      = {2013},
  doi       = {10.1007/978-3-642-17339-4}
}

@book{NeumannW2010,
  author    = {F. Neumann and C. Witt},
  title     = {Bioinspired Computation in Combinatorial Optimization: Algorithms and Their Computational Complexity},
  publisher = {Springer},
  year      = {2010},
  doi       = {10.1007/978-3-642-16544-3}
}

@incollection{Doerr2020,
  author    = {B. Doerr},
  title     = {Probabilistic Tools for the Analysis of Randomized Optimization Heuristics},
  booktitle = {Theory of Evolutionary Computation: Recent Developments in Discrete Optimization},
  pages     = {1--87},
  publisher = {Springer},
  year      = {2020},
  doi       = {10.1007/978-3-030-29414-4_1}
}

@book{Rudolph1997,
  author    = {G. Rudolph},
  title     = {Convergence Properties of Evolutionary Algorithms},
  publisher = {Kovač},
  year      = {1997}
}

@book{NocedalWright2006,
  author    = {J. Nocedal and S. J. Wright},
  title     = {Numerical Optimization},
  edition   = {2nd},
  publisher = {Springer},
  address   = {New York, NY},
  series    = {Springer Series in Operations Research and Financial Engineering},
  year      = {2006},
  doi       = {10.1007/978-0-387-40065-5}
}

@book{Milnor1963,
  author    = {J. W. Milnor},
  title     = {Morse Theory},
  series    = {Annals of Mathematics Studies},
  number    = {51},
  publisher = {Princeton University Press},
  address   = {Princeton, NJ},
  year      = {1963},
  note      = {Based on lecture notes by M. Spivak and R. Wells}
}

@book{Kolmogorov1933,
  author    = {A. N. Kolmogorov},
  title     = {Grundbegriffe der Wahrscheinlichkeitsrechnung},
  publisher = {Springer-Verlag},
  year      = {1933},
  note      = {English translation: \emph{Foundations of the Theory of Probability}, Chelsea, 1956}
}

@book{Billingsley1995,
  author    = {P. Billingsley},
  title     = {Probability and Measure},
  edition   = {3rd},
  publisher = {Wiley},
  series    = {Wiley Series in Probability and Statistics},
  year      = {1995}
}

@book{Billingsley1986,
  author    = {P. Billingsley},
  title     = {Probability and Measure},
  edition   = {2nd},
  publisher = {Wiley},
  year      = {1986}
}

@book{Durrett2019,
  author    = {R. Durrett},
  title     = {Probability: Theory and Examples},
  edition   = {5th},
  publisher = {Cambridge University Press},
  year      = {2019},
  doi       = {10.1017/9781108591034}
}

@book{Kallenberg2021,
  author    = {O. Kallenberg},
  title     = {Foundations of Modern Probability},
  edition   = {3rd},
  publisher = {Springer},
  address   = {Cham},
  series    = {Probability Theory and Stochastic Modelling},
  volume    = {99},
  year      = {2021},
  doi       = {10.1007/978-3-030-61871-1}
}

@book{Dudley2002,
  author    = {R. M. Dudley},
  title     = {Real Analysis and Probability},
  edition   = {2nd},
  publisher = {Cambridge University Press},
  series    = {Cambridge Studies in Advanced Mathematics},
  year      = {2002},
  doi       = {10.1017/CBO9780511755347}
}

@book{Dudley1989,
  author    = {R. M. Dudley},
  title     = {Real Analysis and Probability},
  publisher = {Wadsworth \& Brooks/Cole},
  year      = {1989}
}

@book{LevinPeres2017,
  author    = {D. A. Levin and Y. Peres},
  title     = {Markov Chains and Mixing Times},
  edition   = {2nd},
  publisher = {American Mathematical Society},
  year      = {2017}
}

@book{MeynTweedie2009,
  author    = {S. Meyn and R. L. Tweedie},
  title     = {Markov Chains and Stochastic Stability},
  edition   = {2nd},
  publisher = {Cambridge University Press},
  year      = {2009},
  doi       = {10.1017/CBO9780511626630}
}

@article{IonescuTulcea1949,
  author    = {C. T. Ionescu Tulcea},
  title     = {Mesures dans les espaces produits},
  journal   = {Atti Accad. Naz. Lincei Rend. Cl. Sci. Fis. Mat. Nat. (8)},
  volume    = {7},
  pages     = {208--211},
  year      = {1949}
}

@misc{Shalizi2006IonescuTulcea,
  author    = {C. R. Shalizi},
  title     = {Building Infinite Processes from Regular Conditional Probability Distributions},
  howpublished = {Lecture notes, 36-754 (CMU), Lecture~3},
  year      = {2006},
  note      = {See Section~3.1 (Definition~30: probability kernel) and Section~3.2 (Theorem~33: Ionescu--Tulcea extension theorem)}
}

@misc{Bayer2011AdvancedProb,
  author    = {C. Bayer},
  title     = {Advanced Probability Theory},
  howpublished = {Lecture notes},
  year      = {2011},
  note      = {See Theorem~5.14 (Ionescu--Tulcea) and Theorem~5.16 (Kolmogorov extension theorem)}
}

@book{CoxOakes1984,
  author    = {D. R. Cox and D. Oakes},
  title     = {Analysis of Survival Data},
  publisher = {Chapman and Hall},
  year      = {1984}
}

@techreport{CEC2017,
  author      = {N. H. Awad and M. Z. Ali and J. J. Liang and B. Y. Qu and P. N. Suganthan},
  title       = {Problem Definitions and Evaluation Criteria for the {CEC} 2017 Special Session and Competition on Single Objective Real-Parameter Numerical Optimization},
  institution = {Nanyang Technological University},
  year        = {2016}
}

@article{MolgaSmutnicki2005,
  author    = {M. Molga and C. Smutnicki},
  title     = {Test Functions for Optimization Needs},
  year      = {2005},
  note      = {Available online}
}

@article{DrusvyatskiyLewis2013,
  author    = {D. Drusvyatskiy and A. S. Lewis},
  title     = {Tilt Stability, Uniform Quadratic Growth, and Strong Metric Regularity of the Subdifferential},
  journal   = {SIAM Journal on Optimization},
  volume    = {23},
  number    = {1},
  pages     = {256--267},
  year      = {2013},
  doi       = {10.1137/120876551}
}

@article{DrusvyatskiyLewis2018,
  author    = {D. Drusvyatskiy and A. S. Lewis},
  title     = {Error Bounds, Quadratic Growth, and Linear Convergence of Proximal Methods},
  journal   = {Mathematics of Operations Research},
  volume    = {43},
  number    = {3},
  pages     = {919--948},
  year      = {2018},
  doi       = {10.1287/moor.2017.0889}
}

@book{RockafellarWets1998,
  author    = {R. T. Rockafellar and R. J.-B. Wets},
  title     = {Variational Analysis},
  publisher = {Springer},
  series    = {Grundlehren der mathematischen Wissenschaften},
  volume    = {317},
  year      = {1998},
  doi       = {10.1007/978-3-642-02431-3}
}

@article{FusekKlatte2013,
  author    = {P. Fusek and D. Klatte and W. Li},
  title     = {The {Lipschitz} Modulus of the Argmin Mapping in Linear Semi-Infinite Optimization},
  journal   = {SIAM Journal on Optimization},
  volume    = {23},
  number    = {1},
  pages     = {17--37},
  year      = {2013},
  doi       = {10.1137/110852917}
}

@article{Zhang2009JADE,
  author    = {J. Zhang and A. C. Sanderson},
  title     = {{JADE}: Adaptive Differential Evolution with Optional External Archive},
  journal   = {IEEE Transactions on Evolutionary Computation},
  volume    = {13},
  number    = {5},
  pages     = {945--958},
  year      = {2009},
  doi       = {10.1109/TEVC.2009.2014613}
}

@article{Brest2006jDE,
  author    = {J. Brest and S. Greiner and B. Bošković and M. Mernik and V. Žumer},
  title     = {Self-Adapting Control Parameters in Differential Evolution: A Comparative Study on Numerical Benchmark Problems},
  journal   = {IEEE Transactions on Evolutionary Computation},
  volume    = {10},
  number    = {6},
  pages     = {646--657},
  year      = {2006},
  doi       = {10.1109/TEVC.2006.872133}
}

@inproceedings{Qin2005SaDE,
  author    = {A. K. Qin and P. N. Suganthan},
  title     = {Self-adaptive Differential Evolution Algorithm for Numerical Optimization},
  booktitle = {Proceedings of the IEEE Congress on Evolutionary Computation (CEC)},
  pages     = {1785--1791},
  year      = {2005},
  doi       = {10.1109/CEC.2005.1554904}
}

\end{document}